\documentclass[11pt,a4paper]{article}
\usepackage[hyperref]{emnlp2020}
\usepackage{times}
\usepackage{latexsym}

\usepackage{microtype}

\aclfinalcopy 

\usepackage{amsmath}
\usepackage{amssymb}
\usepackage{mathrsfs}
\usepackage{xcolor}
\usepackage{amsthm}
\usepackage{amsfonts}
\usepackage{graphicx}
\usepackage{stmaryrd}
\usepackage{algorithm}
\usepackage{xr}

\usepackage{times}
\usepackage{latexsym}
\usepackage{url}
\usepackage{array}
\usepackage{booktabs}
\usepackage{balance}
\usepackage{colortbl}
\usepackage{makecell}
\usepackage{soul}
\usepackage{blindtext}
\usepackage{multirow}
\usepackage{subcaption}
\usepackage{amssymb}
\usepackage{pifont}

\newcolumntype{P}[1]{>{\centering\arraybackslash}p{#1}}

\usepackage{microtype}

\newcommand{\tbf}[1]{\mathbf{#1}}

\theoremstyle{definition}

\theoremstyle{remark}

\theoremstyle{plain}
\newtheorem{theorem}{Theorem}[section]

\theoremstyle{plain}

\theoremstyle{plain}

\newtheorem{proposition}[theorem]{Proposition}
\newtheorem{lemma}[theorem]{Lemma}

\usepackage{graphicx,wrapfig,lipsum}

\usepackage{amsmath,amsfonts,bm}

\def\eqref#1{(\ref{#1})}

\def\1{\bm{1}}

\def\mbtheta{\mathbf{\theta}}

\def\rmI{{\mathbf{I}}}

\def\rmV{{\mathbf{V}}}
\def\rmW{{\mathbf{W}}}

\def\rmY{{\mathbf{Y}}}
\def\rmZ{{\mathbf{Z}}}

\def\vzero{{\bm{0}}}

\def\vtheta{{\bm{\theta}}}
\def\va{{\bm{a}}}
\def\vb{{\bm{b}}}

\def\vh{{\bm{h}}}

\def\vk{{\bm{k}}}

\def\vp{{\bm{p}}}
\def\vq{{\bm{q}}}

\def\vs{{\bm{s}}}

\def\vu{{\bm{u}}}
\def\vv{{\bm{v}}}

\def\vx{{\bm{x}}}
\def\vy{{\bm{y}}}
\def\vz{{\bm{z}}}

\def\mK{{\bm{K}}}

\def\mV{{\bm{V}}}
\def\mW{{\bm{W}}}
\def\mX{{\bm{X}}}
\def\mY{{\bm{Y}}}
\def\mZ{{\bm{Z}}}

\DeclareMathAlphabet{\mathsfit}{\encodingdefault}{\sfdefault}{m}{sl}
\SetMathAlphabet{\mathsfit}{bold}{\encodingdefault}{\sfdefault}{bx}{n}

\def\sN{{\mathbb{N}}}

\def\sQ{{\mathbb{Q}}}

\newcommand{\bt}{\bar{t}}
\newcommand{\minn}{\overline}

\usepackage{commath}
\usepackage{amssymb}
\usepackage{amsmath,amsfonts,bm}

\newcommand{\bn}{\mathbb{N}}

\newcommand{\bz}{\mathbb{Z}}
\newcommand{\bq}{\mathbb{Q}}

\newcommand{\mx}{\mathbf{X}}

\newcommand{\mk}{\mathbf{K}}
\newcommand{\mv}{\mathbf{V}}

\newcommand{\startsym}{\ensuremath \#}
\newcommand{\termsym}{\ensuremath \$}

\newcommand{\prop}{\bm{\omega}}
\newcommand{\diff}{\bm{\delta}}

\newcommand{\sigmoid}{\sigma}
\newcommand{\gam}{\bm{\gamma}}
\newcommand{\ta}{\bm{\tau}}
\newcommand{\Gam}{\bm{\Gamma}}
\newcommand{\stack}{\Psi}
\newcommand{\stacka}{\Psi^{(1)}}
\newcommand{\stackb}{\Psi^{(2)}}
\newcommand{\cantor}{\mathcal{C}_4}

\newcommand{\brac}[1]{\ensuremath(#1)}
\newcommand{\bl}{\ensuremath (\ell)}

\newcommand{\vzd}{\vzero_d}
\newcommand{\vzh}{\vzero_h}
\newcommand{\vzs}{\vzero_s}
\newcommand{\vzw}{\vzero_{\omega}}

\newcommand{\vke}{\vk^e}
\newcommand{\vve}{\vv^e}
\newcommand{\oh}[1]{\llbracket #1\rrbracket}
\newcommand{\freq}{\phi}

\newtheorem{claim}{Claim}

\newcommand{\Att}{\operatorname{Att}}
\newcommand{\Enc}{\operatorname{Enc}}
\newcommand{\TEnc}{\operatorname{TEnc}}
\newcommand{\Dec}{\operatorname{Dec}}
\newcommand{\TDec}{\operatorname{TDec}}
\newcommand{\Trans}{\operatorname{Trans}}
\newcommand{\DTrans}{\operatorname{DTrans}}

\newcommand{\pos}{\ensuremath\mathrm{pos}}

\newcommand{\RNN}{\ensuremath\mathrm{RNN}}
\newcommand{\<}{\ensuremath \langle}
\renewcommand{\>}{\ensuremath \rangle}
\newcommand{\N}{\sN}
\newcommand{\Q}{\sQ}
\newcommand{\att}{\mathrm{att}}
\newcommand{\mrsymb}{\mathrm{symb}}
\newcommand{\symb}{\beta}
\newcommand{\alphabetsize}{m} 
\newcommand{\fatt}{f^{\att}}
\newcommand{\hardmax}{\ensuremath{\mathsf{hardmax}}}
\newcommand{\softmax}{\ensuremath{\mathsf{softmax}}}

\newcommand{\mbQ}{\mathbb{Q}}
\newcommand{\mbN}{\mathbb{N}}
\newcommand{\indicator}{\mathbb{I}}
\newcommand{\btheta}{\bm{\theta}}
\newcommand{\mDelta}{\bm{\Delta}}
\newcommand{\BaseEmbedding}{f_b}

\title{On the Computational Power of Transformers and its \\ Implications in Sequence Modeling}

\author{Satwik Bhattamishra\quad Arkil Patel\quad Navin Goyal\\
	Microsoft Research India\\
	{\tt \small \{t-satbh,t-arkpat,navingo\}@microsoft.com} \\
}

\date{}

\begin{document}
\maketitle
\begin{abstract}

Transformers are being used extensively across several sequence modeling tasks. Significant research effort has been devoted to experimentally probe the inner workings of Transformers. However, our conceptual and theoretical understanding of their power and inherent limitations is still nascent. In particular, the roles of various components in Transformers such as positional encodings, attention heads, residual connections, and feedforward networks, are not clear. In this paper, we take a step towards answering these questions. We analyze the computational power as captured by Turing-completeness. We first provide an alternate and simpler proof to show that vanilla Transformers are Turing-complete and then we prove that Transformers with only positional masking and without any positional encoding are also Turing-complete. We further analyze the necessity of each component for the Turing-completeness of the network; interestingly, we find that a particular type of residual connection is necessary. We demonstrate the practical implications of our results via experiments on machine translation and synthetic tasks. 

\end{abstract}



\section{Introduction}
Transformer \cite{vaswani2017attention} is a recent self-attention based sequence-to-sequence architecture which has led to state of the art results across various NLP tasks including machine translation \cite{ott-etal-2018-scaling}, language modeling \cite{radford2018improving} and question answering \cite{devlin-etal-2019-bert}. Although a number of variants of Transformers have been proposed, the original architecture  still underlies these variants.


While the training and generalization of machine learning models such as Transformers are the central goals in their analysis, an essential prerequisite to this end is characterization of the computational power of the model: training a model for a certain task cannot succeed if the model is computationally incapable of carrying out the task. While the computational capabilities of recurrent networks (RNNs) have been studied for decades \cite{kolen2001field,siegelmann2012neural},
 for Transformers we are still in the early stages. 

\begin{figure}[t]
	\centering
	\includegraphics[scale=0.28]{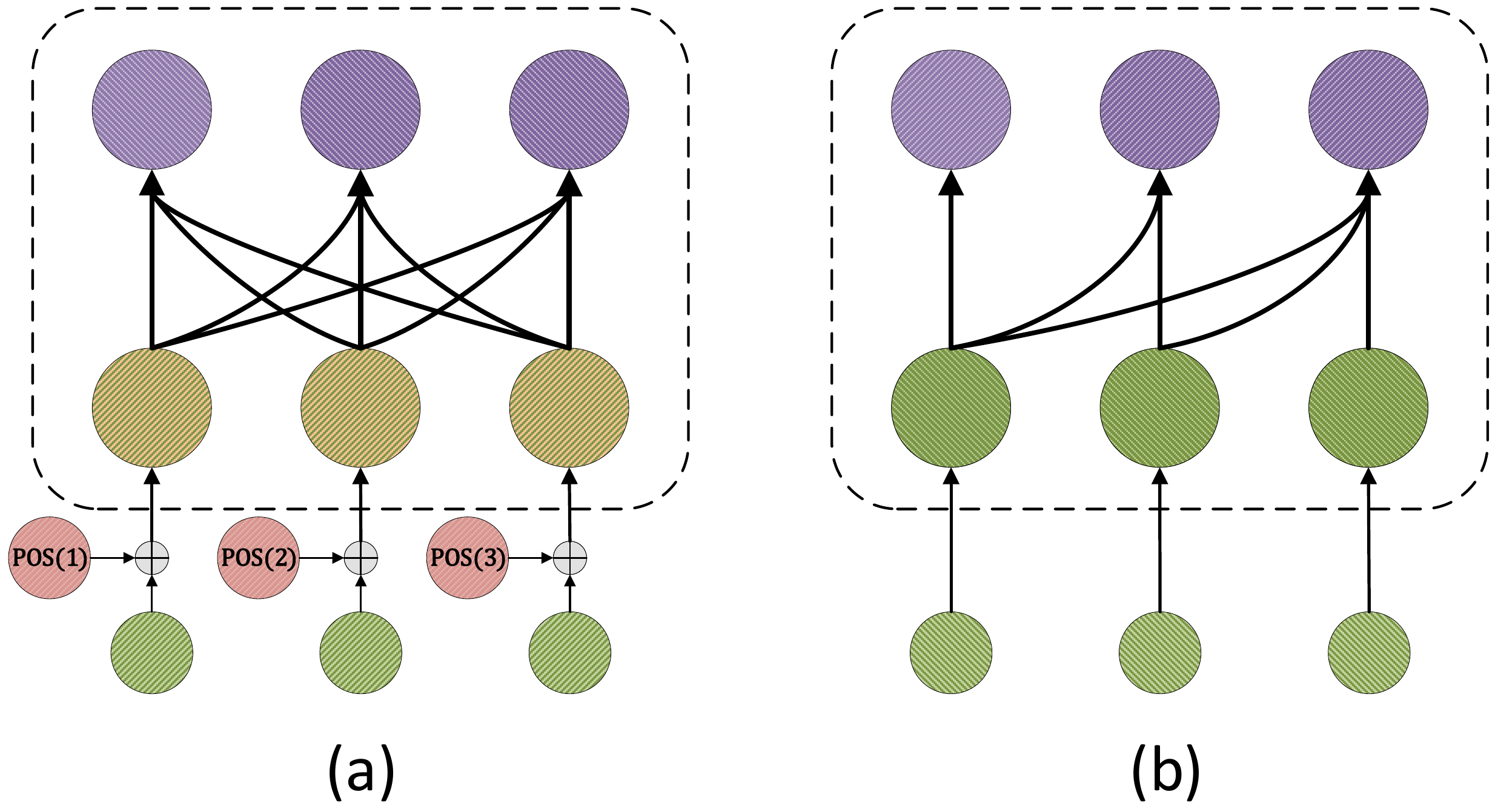}
	\caption{\label{fig:intro} (a) Self-Attention Network with positional 
		encoding, (b) Self-Attention Network with positional masking without any positional encoding}
\end{figure}


 The celebrated work of \citet{siegelmann1992computational} showed, assuming arbitrary precision, that RNNs are \emph{Turing-complete}, meaning that they are capable of carrying out any algorithmic task formalized by Turing machines. 
Recently, \citet{perez2019turing} have shown that vanilla Transformers with hard-attention can also simulate Turing machines given arbitrary precision. However, in contrast to RNNs, Transformers consist of several components and it is unclear which components are necessary for its Turing-completeness and thereby crucial to its computational expressiveness.

The role of various components of the Transformer in its efficacy is an important question for further improvements. Since the Transformer does not process the input sequentially, it requires some form of positional information. Various positional encoding schemes have been proposed to capture order information \cite{shaw-etal-2018-self,dai-etal-2019-transformer,Huang2018AnIR}. At the same time, on machine translation, \citet{yang2019assessing} showed that  the performance of Transformers with only positional masking \cite{shen2018disan} is comparable to that with positional encodings. In case of positional masking (Fig. \ref{fig:intro}), as opposed to explicit encodings, the model is only allowed to attend over preceding inputs and no additional positional encoding vector is combined with the input vector. \citet{kerneltsai2019transformer} raised the question of whether explicit encoding is necessary if positional masking is used. Additionally, since \citet{perez2019turing}'s Turing-completeness proof relied heavily on residual connections, they asked whether these connections are essential for Turing-completeness. In this paper, we take a step towards answering such questions.

\noindent Below, we list the main contributions of the paper,
\begin{itemize}
	\item We provide an alternate and arguably simpler proof to show that Transformers are Turing-complete by directly relating them to RNNs.
	\item More importantly, we prove that Transformers with positional masking and without positional encoding are also Turing-complete.
	\item We analyze the necessity of various components such as self-attention blocks, residual connections and feedforward networks for  Turing-completeness. Figure \ref{fig:transformer} provides an overview.
	\item  We explore implications of our results on machine translation and synthetic tasks.\footnote{We have made our source code available at \href{https://github.com/satwik77/Transformer-Computation-Analysis}{https://github.com/satwik77/Transformer-Computation-Analysis}.}
\end{itemize}

\section{Related Work}
%

\noindent \textbf{Computational Power of neural networks} has been studied since the foundational paper \citet{mcculloch1943logical}; in particular, among sequence-to-sequence models,
this aspect of RNNs has long been studied \cite{kolen2001field}. The seminal work by \citet{siegelmann1992computational} showed that 
RNNs can simulate a Turing machine by using unbounded precision. \citet{chen2017recurrent} showed that RNNs with ReLU activations are also Turing-complete.  Many recent works have explored the computational power of RNNs in practical settings. Several works \cite{merrill-etal-2020-formal}, \cite{weiss2018practical} recently studied the ability of RNNs to recognize counter-like languages. The capability of RNNs to recognize strings of balanced parantheses has also been studied \cite{sennhauser-berwick-2018-evaluating,skachkova-etal-2018-closing}. However, such analysis on Transformers has been scarce.

\noindent \textbf{Theoretical work on Transformers} was initiated by \citet{perez2019turing} who formalized the notion of Transformers and showed that it can simulate a Turing machine given arbitrary precision. Concurrent to our work, there have been several efforts to understand self-attention based models \cite{levine2020limits,kim2020lipschitz}. \citet{hron2020infinite} show that Transformers behave as Gaussian processes when the number of heads tend to infinity. \citet{hahn2019theoretical} showed some 
limitations of Transformer encoders in modeling regular and context-free languages. It has been recently shown that Transformers are universal approximators of sequence-to-sequence functions given arbitrary precision \cite{yun2019transformers}. However, these are not applicable\footnote{\citet{hahn2019theoretical} and \citet{yun2019transformers} study encoder-only seq-to-seq models with fixed length outputs in which the computation halts as soon as the last symbol of the input is processed. Our work is about the full Transformer (encoder and decoder) which is a seq-to-seq model with variable length sequence output in which the decoder starts operating sequentially after the encoder.} to the complete Transformer architecture. With a goal similar to ours, \citet{kerneltsai2019transformer} attempted to study the attention mechanism via a kernel formulation. However, a systematic study of various components of Transformers has not been done.



\section{Definitions and Preliminaries}

All the numbers used in our computations will be from the set of rational numbers denoted $\mbQ$. For a sequence $\mX = (\vx_{1},\ldots, \vx_n)$, we set $\mX_j := (\vx_{1},\ldots, \vx_j)$ for $1 \leq j \leq n$.
We will work with an alphabet $\Sigma$ of size $m$, with special symbols $\#$ and $\$$ signifying the beginning and end of the input sequence, respectively. The symbols are mapped to vectors via a given `base' embedding $\BaseEmbedding: \Sigma \rightarrow \bq^{d_b}$, where $d_b$ is the dimension of the embedding. E.g., this embedding could be the one used for processing the symbols by the RNN. 

We set $\BaseEmbedding(\startsym) = \vzero_{d_b} \text{ and } \BaseEmbedding(\$) = \vzero_{d_b}$. 
\emph{Positional encoding} is a function $\pos: \mbN \to \mbQ^{d_b}$. Together, these provide embedding for a symbol $s$ at position $i$ given by 
$f(\BaseEmbedding(s), \pos(i))$, often taken to be simply $\BaseEmbedding(s) + \pos(i)$. Vector $\oh{s} \in \mbQ^m$ denotes one-hot encoding of a symbol $s \in \Sigma$.


\subsection{RNNs}\label{subsec:def_rnn}

We follow \citet{siegelmann1992computational} in our definition of RNNs. To feed the sequences $s_1 s_2 \ldots s_n \in \Sigma^*$ to the RNN, these are converted to the vectors 
$\vx_1, \vx_2, \ldots, \vx_n$ where $\vx_i = \BaseEmbedding(s_i)$. The RNN is given by the recurrence  $\vh_t = g(  \mW_h \vh_{t-1} + \mW_x \vx_t + \vb)$, where
$t \geq 1$, function $g(\cdot)$ is a multilayer feedforward network (FFN) with activation  $\sigmoid$, bias vector $\vb \in \mbQ^{d_h}$, 
matrices $\mW_h \in \mbQ^{d_h \times d_h}$ and $\mW_x \in \mbQ^{d_h \times d_b}$, and $\vh_t \in \mbQ^{d_h}$ is the hidden state with given initial hidden state $\vh_0$; $d_h$ is the hidden state dimension. 

After the last symbol $s_n$ has been fed, we continue to feed the RNN with the terminal symbol $\BaseEmbedding(\$)$ until it halts. This allows the RNN
to carry out computation after having read the input. 

A class of seq-to-seq neural networks is Turing-complete if the class of languages recognized by the networks is exactly the class of languages recognized by Turing machines.

\begin{theorem}\label{th:TC_RNN_main}\cite{siegelmann1992computational}
Any seq-to-seq function $\Sigma^* \to \Sigma^*$ computable by a Turing machine can also be computed by an RNN.  
\end{theorem}
\noindent For details please see section \ref{subsec:ssrnn} in appendix.

\subsection{Transformer Architecture}
\noindent\textbf{Vanilla Transformer.} We describe the original Transformer architecture with positional encoding  \cite{vaswani2017attention} as formalized by \citet{perez2019turing}, with
some modifications. All vectors in this subsection are from $\mbQ^d$.


The transformer, denoted $\Trans$, is a seq-to-seq architecture. Its input consists of (i) a sequence $\mX = (\vx_1,\ldots, \vx_n)$  of vectors, (ii) a seed vector $\vy_0$. The output is a sequence $\mY=(\vy_1,\ldots, \vy_r)$ of vectors. 
The sequence $\mX$ is obtained from the sequence $(s_1, \ldots, s_n) \in  \Sigma^{n}$ of symbols by using the embedding mentioned earlier: 
$\vx_i = f(\BaseEmbedding(s_i), \pos(i))$.

The transformer consists of composition of \emph{transformer encoder} and \emph{transformer decoder}. 
For the feedforward networks in the transformer layers we use the activation as in \citet{siegelmann1992computational}, namely the saturated linear activation function $\sigmoid(x)$ which takes value $0$ for $x<0$, value $x$ for $0 < x < 1$ and value $1$ for $x > 1$. This activation can be easily replaced by the standard ReLU activation via $\sigmoid(x) = \mathrm{ReLU}(x)-\mathrm{ReLU}(x-1)$.


\noindent\textbf{Self-attention.} The self-attention mechanism takes as input (i) a \emph{query} vector
$\vq$, (ii) a sequence of \emph{key} vectors $\mK = (\vk_1, \ldots , \vk_n)$, and (iii) a sequence of \emph{value} vectors $\mV = (\vv_1, \ldots , \vv_n)$. 
The $\vq$-attention over $\mK$ and $\mV$, denoted $\Att(\vq, \mK, \mV)$, is a vector $\va = \alpha_1\vv_1 + \alpha_2\vv_2 + \cdots +\alpha_n\vv_n$, where 
{(i)} $(\alpha_1, \ldots , \alpha_n) = \rho (\fatt(\vq, \vk_1), \ldots , \fatt(\vq, \vk_n))$.\\
\noindent {(ii)} The normalization function $\rho: \mbQ^n \to \mbQ_{\geq 0}^n$ is $\hardmax$: for $\vx = (x_1, \ldots, x_n) \in \mbQ^n$, if the maximum value occurs $r$ times among  $x_1, \ldots, x_n$, then 
$\hardmax(\vx)_i := 1/r$ if $x_i$ is a maximum value and $\hardmax(\vx)_i := 0$ otherwise. In practice, the $\softmax$ is often used but its output values are in general not rational. \\
\noindent {(iii)} For vanilla transformers, the scoring function $f^\att$ used is a combination of multiplicative attention \cite{vaswani2017attention} and a non-linear function: $\fatt(\vq, \vk_i) = -\left | \<\vq, \vk_i \> \right |$. This was also used by \citet{perez2019turing}. \\
\noindent \textbf{Transformer encoder.} A \emph{single-layer encoder} is a function $\Enc(\mX; \btheta)$, 
with input $\mX = (\vx_1, \ldots, \vx_n)$ a sequence of vectors in $\mbQ^d$, and parameters $\btheta$. The output is another sequence $\mZ = (\vz_1, \ldots, \vz_n)$ of vectors
in $\mbQ^d$. The parameters $\btheta$ specify functions $Q(\cdot), K(\cdot), V(\cdot)$, and $O(\cdot)$, all of type $\mbQ^d \to \mbQ^d$. The functions
$Q(\cdot), K(\cdot),$ and $ V(\cdot)$ are linear transformations and $O(\cdot)$ an FFN. 
For $1 \leq i \leq n$, the output of the self-attention block is produced by 
\begin{equation}
\va_i = \Att(Q(\vx_i), K(\mX), V(\mX)) + \vx_i\label{eq:enc-enc-att-trans-main}
\end{equation}
This operation is also referred to as the encoder-encoder attention block. The output $\mZ$ is computed by  $\vz_i = O(\va_i) + \va_i$ for $1 \leq i \leq n$.  The addition operations $+\vx_i$ and $+\va_i$ are the  residual connections. The complete $L$-layer transformer encoder $\TEnc^{\brac{L}}(\mX; \btheta)= (\mK^e, \mV^e)$ has the same input $\mX = (\vx_1, \ldots, \vx_n)$ as the single-layer encoder. 
In contrast, its output  $\mK^e = (\vke_1, \ldots, \vke_n)$ and $\mV^e = (\vv^e_1, \ldots \vv^e_n)$ contains two sequences.
$\TEnc^{\brac{L}}$ is obtained by composition of $L$ single-layer encoders: let $\mX^{\brac{0}} := \mX $, and for $0\leq \ell \leq L-1$, let 
$\mX^{\brac{\ell+1}} = \Enc(\mX^{\brac{\ell}}; \btheta_\ell)$ and finally, $\mK^e = K^{\brac{L}}(\mX^{\brac{L}}), \quad \mV^e = V^{\brac{L}}(\mX^{\brac{L}}).$ \\
\noindent\textbf{Transformer decoder.} 
%
The input to a \emph{single-layer decoder}  is (i) $(\mK^e, \mV^e)$ output by the encoder, and (ii) sequence $\mY = (\vy_1, \ldots, \vy_k)$ of vectors for $k \geq 1$. The output is another sequence $\mZ = (\vz_1, \ldots, \vz_k)$.

Similar to the single-layer encoder, a single-layer decoder is parameterized by functions $Q(\cdot), K(\cdot), V(\cdot)$ and $O(\cdot)$ and is defined by 
\begin{eqnarray}
\vp_t & = & \Att(Q(\vy_t),K(\mY_t),V(\mY_t)) 
+ \vy_t, \label{eq:dec-dec-main} \\
\va_t & = & \Att(\vp_t,\mK^e,\mV^e) 
+ \vp_t, \label{eq:enc-dec-main} \\
\vz_t & = & O(\va_t) + \va_t, \nonumber 
\end{eqnarray}
where $1 \leq t \leq k$. The operation in (\ref{eq:dec-dec-main}) will be referred to as the \emph{decoder-decoder attention} block and the operation in (\ref{eq:enc-dec-main}) as the
\emph{decoder-encoder attention} block. In (\ref{eq:dec-dec-main}), positional masking is applied to prevent the network from attending over symbols which are ahead of them.

An $L$-layer Transformer decoder $\TDec^L((\mK^e, \mV^e), \mY; \btheta) = \vz$ is obtained by repeated application of $L$ single-layer decoders each with its own parameters, and a transformation function $F : \bq^d \rightarrow \bq^d$ applied to the last vector in the sequence of vectors output by the final decoder. 
Formally, for $0 \leq \ell \leq L-1$ and $\mY^0 := \mY$ we have
$\mY^{\ell+1} = \Dec((\mK^e, \mV^e), \mY^\ell; \btheta_\ell), \quad \vz = F(\vy_{k}^L).$ 
Note that while the output of a single-layer decoder is a sequence of vectors, the output of an $L$-layer Transformer decoder is a single vector.


\noindent\textbf{The complete Transformer.}
The output $\Trans(\mX,\vy_0) = \mY$ is computed by the recurrence
$\tilde{\vy}_{t+1}  =  \TDec(\TEnc(\mX), (\vy_0,\vy_1,\ldots,\vy_{t}))$,  for $0 \leq t \leq r-1$.
We get $\vy_{t+1}$ by adding positional encoding: $\vy_{t+1} = \tilde{\vy}_{t+1} + \pos(t+1)$. 

\noindent\textbf{Directional Transformer.} 
We denote the Transformer with only positional masking and no positional encodings as Directional Transformer and use them interchangeably. In this case, we use standard multiplicative attention as the scoring function in our construction, i.e, $\fatt(\vq, \vk_i) =  \< \vq, \vk_i \>$.
The general architecture is the same as for the vanilla case; the differences due to positional masking are the following. 

There are no positional encodings. So the input vectors $\vx_i$ only involve $\BaseEmbedding(s_i)$. 
Similarly, $\vy_t = \Tilde{\vy}_t$. 
In (\ref{eq:enc-enc-att-trans-main}), $\Att(\cdot)$ is replaced by 
$\Att(Q(\vx_i), K(\mX_i), V(\mX_i))$ where $\mX_i := (\vx_{1},\ldots, \vx_i)$ for $1 \leq i \leq n$.
Similarly, in (\ref{eq:enc-dec-main}), $\Att(\cdot)$ is replaced by
$\Att(\vp_t,\mK^{e}_{t},\mV^{e}_{t})$.  \\
\noindent\textbf{Remark 1.} Our definitions deviate slightly from practice, hard-attention being the main one since $\hardmax$ keeps the values rational whereas $\softmax$ takes the values to irrational space. Previous studies have shown that soft-attention behaves like hard-attention in practice and \citet{hahn2019theoretical} discusses its practical relevance. 

\noindent \textbf{Remark 2.} Transformer Networks with positional encodings are not necessarily equivalent in terms of their computational expressiveness \cite{yun2019transformers} to those with only positional masking when considering the encoder only model (as used in BERT and GPT-2). Our results in Section \ref{subsec:tc_res} show their equivalence in terms of expressiveness for the complete seq-to-seq architecture.

\section{Primary Results}\label{sec:results}

\subsection{Turing-Completeness Results}\label{subsec:tc_res}

In light of Theorem~\ref{th:TC_RNN_main}, to prove that Transformers are Turing-complete, it suffices to show that they can \emph{simulate} RNNs. We say that a Transformer simulates an RNN (as defined in Sec.~\ref{subsec:def_rnn}) if on every input $s \in \Sigma^*$, at each step $t$, the vector 
$\vy_{t}$ contains the hidden state $\vh_t$ as a subvector, i.e. $\vy_{t} = [\vh_t, \cdot]$, and halts at the same step as the RNN. 

\begin{theorem}\label{th:TC_trans_informal}
	The class of Transformers with positional encodings is Turing-complete.
\end{theorem}
\begin{proof}[Proof Sketch]

The input $s_0, \ldots, s_n \in \Sigma^*$ is provided to the transformer as the sequence of vectors $\vx_0, \ldots, \vx_n$, where 
$\vx_i = [  \vzero_{d_h}, \BaseEmbedding(s_i), \vzero_{d_h}, i, 1 ]$, which has as sub-vector the given base embedding $ \BaseEmbedding(s_i)$ and the positional encoding $i$, along with extra coordinates set to constant values and will be used later.

The basic observation behind our construction of the simulating Transformer is that the transformer decoder can naturally implement the recurrence operations of the type used by RNNs. To this end, the FFN $O^{\mathrm{dec}}(\cdot)$ of the decoder, which plays the same role as the FFN component of the RNN, needs sequential access to the input in the same way as RNN. But the Transformer receives the whole input at the same time. We utilize positional encoding along with the attention mechanism to isolate $\vx_t$ at time $t$ and feed it to $O^{\mathrm{dec}}(\cdot)$, thereby simulating the RNN.

As stated earlier, we append the input $s_1, \ldots, s_n$ of the RNN with $\$$'s until it halts. Since the Transformer takes its input all at once, appending by $\$$'s is not possible (in particular, we do not know how long the computation would take). Instead, we append the input with a single $\$$. After encountering a $\$$ once, the Transformer will feed (encoding of) $\$$ to $O^{\mathrm{dec}}(\cdot)$ in subsequent steps until termination.
Here we confine our discussion to the case $t \leq n$; the $t > n$ case is slightly different but simpler.

The construction is straightforward: it has only one head, one encoder layer and one decoder layer; moreover, the attention mechanisms in the encoder and the decoder-decoder attention block of the decoder are trivial as described below.

The encoder attention layer does trivial computation in that it merely computes the identity function: $\vz_i = \vx_i$, which can be easily achieved, e.g. by using the residual connection and setting the value vectors to $\vzero$. The final $K^{\brac{1}}(\cdot)$ and $V^{\brac{1}}(\cdot)$ functions bring $(\mK^e, \mV^e)$ into useful forms by appropriate linear transformations: $	\vk_i  =  [ \vzero_{d_b}, \vzero_{d_b}, \vzero_{d_b}, -1, i ]$ and $\vv_i  =  [  \vzero_{d_b}, \BaseEmbedding(s_i), \vzero_{d_b}, 0, 0 ]$.
Thus, the key vectors only encode the positional information and the value vectors only encode the input symbols.

\begin{figure}[t]
	\centering
	\includegraphics[scale=0.55]{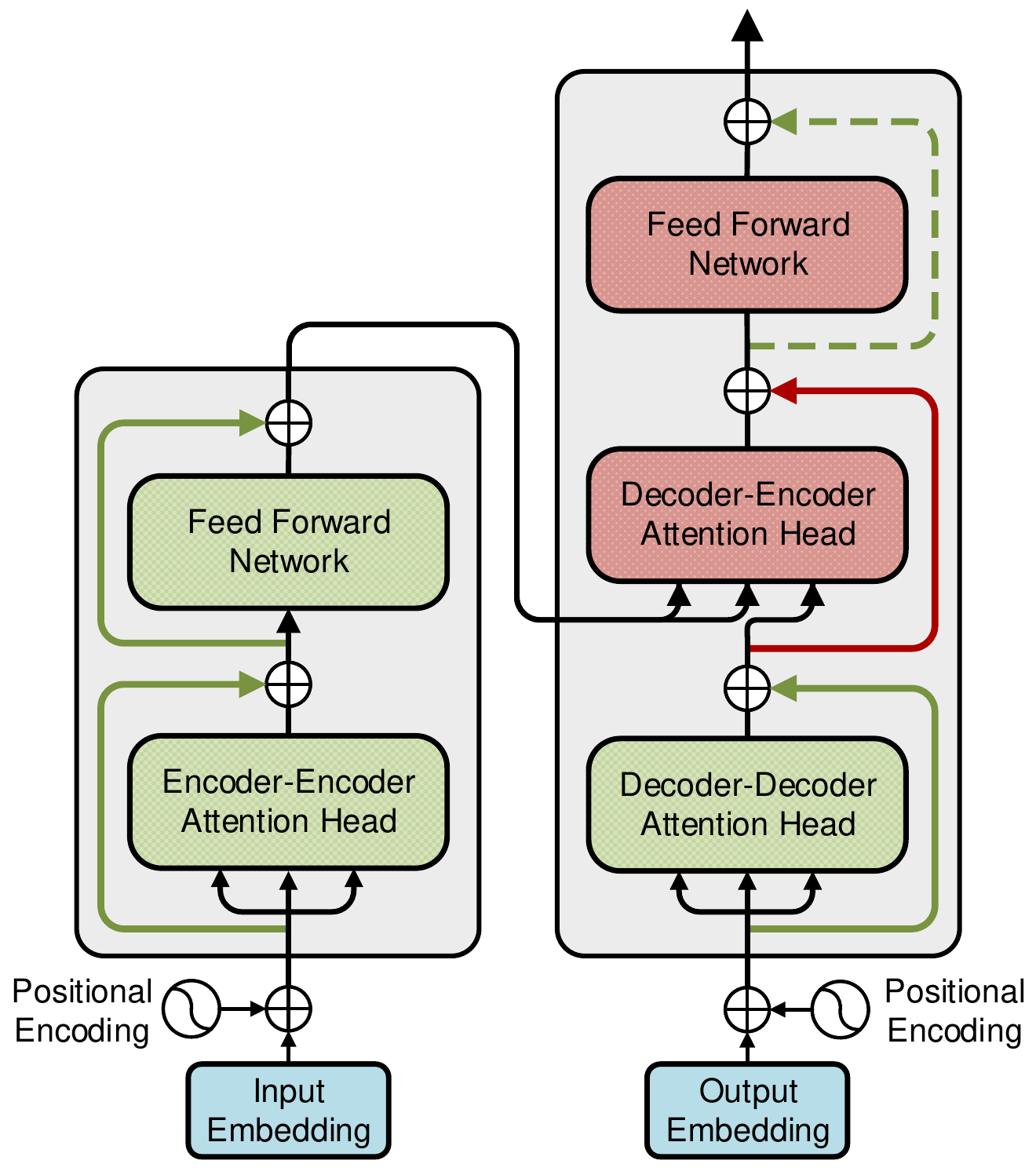}
	\caption{\label{fig:transformer} Transformer network with various 
		components highlighted. The components marked red are essential for the 
		Turing-completeness whereas for the pairs of blocks and residual connections marked green, 
		either one of the component is enough. The dashed residual connection is 
		not necessary for Turing-completeness of the network.}
\end{figure}

The output sequence of the decoder is $\vy_1, \vy_2, \ldots$. Our construction will ensure, by induction on $t$, that $\vy_t$ contains the hidden states $\vh_{t}$ of the RNN as a sub-vector along with positional information: $\vy_t  =  [  \vh_{t}, \vzero_{d_b}, \vzero_{d_b}, t+1, 1 ]$. This is easy to arrange for $t=0$, and assuming it for $t$ we prove it for $t+1$. As for the encoder, the decoder-decoder attention block acts as the identity: $\vp_t = \vy_t$. Now, using the last but one coordinate in $\vy_t$ representing the time $t+1$, the attention mechanism
$\Att(\vp_t, \mK^e, \mV^e)$ can retrieve the embedding of the $t$-th input symbol $\vx_t$. This is possible because in the key vector $\vk_i$ mentioned above, almost all coordinates other than the one representing the position $i$ are set to $0$, allowing the mechanism to only focus on the positional information and not be distracted by the other contents of $\vp_t = \vy_t$: the scoring function has value $f^{\att}(\vp_t, \vk_i) = -|\< \vp_t, \vk_i \>| = -|i-(t+1)|$. For a given $t$, it is maximized at $i = t+1$ for $t < n$ and at $i=n$ for $t \geq n$. This use of scoring function is similar to \citet{perez2019turing}.

At this point, $O^{\mathrm{dec}}(\cdot)$ has at its disposal the hidden state $\vh_{t}$ (coming from $\vy_t$ via $\vp_t$ and the residual connection) and the input symbol $\vx_t$ (coming via the attention mechanism and the residual connection). Hence  $O(\cdot)$ can act just like the FFN (Lemma \ref{lem:rnn_trans_comp}) underlying the RNN to compute $\vh_{t+1}$ and thus $\vy_{t+1}$, proving the induction hypothesis. 
The complete construction can be found in Sec.~\ref{subsec:trans_tc} in the appendix.

\end{proof}

\begin{theorem}\label{th:TC_dirtrans_informal}
	The class of Transformers with positional masking and no explicit positional encodings is Turing-complete.
\end{theorem}
\begin{proof}[Proof Sketch] As before, by Theorem~\ref{th:TC_RNN_main} it suffices to show that Transformers can simulate RNNs.
The input $s_0, \ldots, s_n$ is provided to the transformer as the sequence of vectors $\vx_0, \ldots, \vx_n$, where 
$\vx_i =  [ \vzero_{d_h}, \vzero_{d_h}, \BaseEmbedding(s_i),    \oh{s_i}, 0, \vzero_{m}, \vzero_{m}, \vzero_{m} ]$.
The general goal for the directional case is similar to the vanilla case, namely we would like the FFN $O^{\mathrm{dec}}(\cdot)$ of the decoder to
directly simulate the computation in the underlying RNN. In the vanilla case, positional encoding and the attention mechanism helped us feed input $\vx_t$ at the $t$-th iteration of the decoder to $O^{\mathrm{dec}}(\cdot)$. However, we no longer have explicit positional information in the input $\vx_t$ such as a coordinate with value $t$. The key insight is that we do not need the positional information explicitly to recover $\vx_t$ at step $t$: in our construction, the attention mechanism with masking will recover $\vx_t$ in an indirect manner even though it's not able to ``zero in'' on the $t$-th position. 

Let us first explain this without details of the construction. We maintain in vector $\prop_t \in \bq^{m}$, with a coordinate each for symbols in $\Sigma$, the fraction of times the symbol has occurred up to step $t$. Now, at a step $t \leq n$, for the difference $\prop_{t}-\prop_{t-1}$  (which is part of the query vector), it can be shown easily that only the coordinate corresponding to $s_t$ is positive. Thus after applying the linearized sigmoid $\sigmoid(\prop_{t}-\prop_{t-1})$, we can isolate the coordinate corresponding to $s_t$. Now using this query vector, the (hard) attention mechanism will be able to retrieve the value vectors for all indices $j$ such that $s_j = s_t$ and output their average. Crucially, the value vector for an index $j$ is essentially $\vx_j$ which depends only on $s_j$. Thus, all these vectors are equal to $\vx_t$, and so is their average. This recovers $\vx_t$, which can now be fed to 
$O^{\mathrm{dec}}(\cdot)$, simulating the RNN. 

We now outline the construction and relate it to the above discussion. As before, for simplicity we restrict to the case $t \leq n$. We use only one head,
one layer encoder and two layer decoder. 
The encoder, as in the vanilla case, does very little other than pass information along. The vectors in $(\mK^e, \mV^e)$ are obtained by the trivial attention mechanism followed by simple linear transformations: $\vke_i = [ \vzero_{d_h}, \vzero_{d_h}, \vzero_{d_b},  \oh{s_i}, 0, \vzero_{m}, \vzero_{m}, \vzero_{m}]$ and 
$\vve_i =  [\vzero_{d_h}, \vzero_{d_h}, \BaseEmbedding(s_i), \vzero_{m}, 0,  \vzero_{m}, \oh{s_i}, \vzero_{m} ]$. 

Our construction ensures that at step $t$ we have 
$ \vy_t  = [  \vh_{t-1},  \vzero_{d_h}, \vzero_{d_b},   \vzero_{m}, \frac{1}{2^t}, \vzero_{m}, \vzero_{m}, \prop_{t-1} ]$. As before, the proof is by induction on $t$.

In the first layer of decoder, the decoder-decoder attention block is trivial: $\vp_t^{\brac{1}} = \vy_t$. In the decoder-encoder attention block, we give
equal attention to all the $t+1$ values, which along with $O^{\mathrm{enc}}(\cdot)$, leads to  
$\vz_{t}^{\brac{1}} =  [ \vh_{t-1}, \;\; \vzero_{d_h} ,  \;\; \vzero_{d_b}, 
\diff_{t}, \;\; \frac{1}{2^{t+1}}, \vzero_{m}, \vzero_{m}, \;\;\prop_{t} ]$, where essentially $\diff_{t} = \sigmoid(\prop_{t}-\prop_{t-1})$, except with a change for the last coordinate due to special status of the last symbol $\$$ in the processing of RNN.

In the second layer, the decoder-decoder attention block is again trivial with $\vp_{t}^{\brac{2}} = \vz_{t}^{\brac{1}}$. We remark that in this construction, the scoring function is the standard multiplicative attention \footnote{Note that it is closer to practice than the scoring function $-|\< \vq, \vk \>|$ used in \citet{perez2019turing} and Theorem \ref{th:TC_trans_informal} }. 
Now  $\<\vp_{t}^{\brac{2}}, \vk^e_{j}\> = \<\diff_t, \oh{s_j} \> = \diff_{t, j}$, which is positive if and only if $s_j = s_t$, as mentioned earlier. Thus attention weights in $\Att(\vp_t^{\brac{2}}, \mK^e_{t}, \mV^e_{t})$ satisfy 
$\hardmax(\<\vp_{t}^{\brac{2}}, \vk^e_{1}\>, \ldots, \<\vp_{t}^{\brac{2}}, \vk^e_{t}\>) = \frac{1}{\lambda_{t}}(\indicator(s_0=s_t), \indicator(s_1=s_t), \ldots, \indicator(s_{t}=s_t))$, where $\lambda_t$ is a normalization constant and $\indicator(\cdot)$ is the indicator.  See Lemma \ref{lem:dec-enc-l2-attention} for more details.

At this point, $O^{\mathrm{dec}}(\cdot)$ has at its disposal the hidden state $\vh_{t}$ (coming from $\vz^{\brac{1}}_t$ via $\vp^{\brac{2}}_t$ and the residual connection) and the input symbol $\vx_t$ (coming via the attention mechanism and the residual connection). Hence $O^{\mathrm{dec}}(\cdot)$ can act just like the FFN underlying the RNN to compute $\vh_{t+1}$ and thus $\vy_{t+1}$, proving the induction hypothesis. 

The complete construction can be found in Sec.~\ref{sec:directional} in the Appendix. 

\end{proof}

\textbf{In practice}, \citet{yang2019assessing} found that for NMT, Transformers with only positional masking achieve comparable performance compared to the ones with positional encodings. Similar evidence was found by \citet{kerneltsai2019transformer}. Our proof for directional transformers entails that there is no loss of order information if positional information is only provided in the form of masking. However, we do not recommend using masking as a replacement for explicit encodings. The computational equivalence of encoding and masking given by our results implies that any differences in their performance must come from differences in learning dynamics. 

\subsection{Analysis of Components}

The results for various components follow from our construction in Theorem \ref{th:TC_trans_informal}. Note that in both the encoder and decoder attention blocks, we need to compute the identity function. We can nullify the role of the attention heads by setting the value vectors to zero and making use of only the residual connections to implement the identity function. Thus, even if we remove those attention heads, the model is still Turing-complete. On the other hand, we can remove the residual connections around the attention blocks and make use of the attention heads to implement the identity function by using positional encodings. Hence, either the attention head or the residual connection is sufficient to achieve Turing-completeness. A similar argument can be made for the FFN in the encoder layer: either the residual connection or the FFN is sufficient for Turing-completeness. 
For the decoder-encoder attention head, since it is the only way for the decoder to obtain information about the input, it is necessary for the completeness. The FFN is the only component that can perform computations based on the input and the computations performed earlier via recurrence and hence, the model is not 
Turing-complete without it. Figure \ref{fig:transformer} summarizes the role of different components with respect to the computational expressiveness of the network.

\begin{proposition}
	The class of Transformers without residual connection around the decoder-encoder attention block is not Turing-complete.
\end{proposition}
\begin{proof}[Proof Sketch]
	We confine our discussion to single-layer decoder; the case of multilayer decoder is similar. Without the residual connection, the decoder-encoder attention block produces $\va_t  =  \Att(\vp_t,\mK^e,\mV^e) = \sum_{i=1}^{n}\alpha_{i}\vv^{e}_{i}$ for some $\alpha_i$'s such that $\sum_i^{n}\alpha_{i}=1$. Note that, without residual connection $\va_t$ can take on at most $2^n-1$ values. This is because by the definition of hard attention the vector $(\alpha_1, \ldots, \alpha_n)$ is characterized by the set of zero coordinates and there are at most $2^n-1$ such sets (all coordinates cannot be zero). This restriction on the number of values on $\va_t$ holds regardless of the value of $\vp_t$. If the task requires the network to produce values of $\va_t$ that come from a set with size at least $2^n$, then the network will not be able to perform the task. Here's an example task: given a number $\Delta \in (0, 1)$, the network must produce  numbers $0, \Delta, 2\Delta, \ldots, k \Delta$, where $k$ is the maximum integer such that $k \Delta \leq 1$. If the network receives a single input $\Delta$, then it is easy to see that the vector $\va_t$ will be a constant ($\vv^{e}_1$) at any step and hence the output of the network will also be constant at all steps. Thus, the model cannot perform such a task. If the input is combined with $n-1$ auxiliary symbols (such as $\startsym$ and $\$$), then in the network, each $\va_t$ takes on at most $2^n-1$ values. Hence, the model will be incapable of performing the task if $\Delta < 1/2^n$. Such a limitation does not exist with a residual connection since the vector $\va_t = \sum_{i=1}^{n}\alpha_{i}\vv^{e}_{i} + \vp_t$ can take arbitrary number of values depending on its prior computations in $\vp_t$. For further details, see Sec.~\ref{subsec:residual} in the Appendix.

\end{proof}

\noindent\textbf{Discussion.} It is perhaps surprising that residual connection, originally proposed to assist in the learning ability of very deep networks, plays a vital role in the computational expressiveness of the network. Without it, the model is limited in its capability to make decisions based on predictions in the previous steps. We explore practical implications of this result in section \ref{sec:exp}.

\section{Experiments}\label{sec:exp}
In this section, we explore the practical implications of our results. Our experiments are geared towards answering the following questions:

\noindent \textbf{Q1.} Are there any practical implications of the limitation of Transformers without decoder-encoder residual connections? What tasks can they do  or not do compared to vanilla Transformers?

\noindent \textbf{Q2.} Is there any additional benefit of using positional masking as opposed to absolute positional encoding \cite{vaswani2017attention}?

Although we showed that Transformers without decoder-encoder residual connection are not Turing complete, it does not imply that they are incapable of 
performing all the tasks. Our results suggest that they are limited in their capability to make inferences based on their previous computations, which is required for tasks such as counting and language modeling. However, it can be shown that the model is capable of performing tasks which rely only on information provided at a given step such as copying and mapping. For such tasks, given positional information at a particular step, the model can look up the corresponding input and map it via the FFN. We evaluate these hypotheses via our experiments.

\begin{table}[th]
	\small{\centering
		\begin{tabular}{p{10em}P{5em}P{4em}}
			\toprule
			\textbf{Model} & \textbf{Copy Task}& \textbf{Counting}\\
			\midrule
			Vanilla Transformers  & 100.0 & 100.0 \\
			- Dec-Enc Residual & 99.7 & 0.0 \\
			- Dec-Dec Residual & 99.7 & 99.8\\
			
			\bottomrule
		\end{tabular}
		\caption{\label{tab:data_synthetic}BLEU scores ($\uparrow$) for copy and counting task. Please see Section \ref{sec:exp} for details }
	}
\end{table}

For our experiments on synthetic data, we consider two tasks, namely the \emph{copy task} and the \emph{counting task}. For the copy task, the goal of a model is to reproduce the input sequence. We sample sentences of lengths between 5-12 words from Penn Treebank and create a train-test split of 40k-1k with all sentences belonging to the same range of length. In the counting task, we create a very simple dataset where the model is given one number between 0 and 100 as input and its goal is to predict the next five numbers. Since only a single input is provided to the encoder, it is necessary for the decoder to be able to make inferences based on its previous predictions to perform this task. The benefit of conducting these experiments on synthetic data is that they isolate the phenomena we wish to evaluate. For both these tasks, we compare vanilla Transformer with the one without decoder-encoder residual connection. As a baseline we also consider the model without decoder-decoder residual connection, since according to our results, that connection does not influence the computational power of the model. We implement a single layer encoder-decoder network with only a single attention head in each block.

We then assess the influence of the limitation on Machine Translation which requires a model to do a combination of both mapping and inferring from computations in previous timesteps. We evaluate the models on IWSLT'14 German-English dataset and IWSLT'15 English-Vietnamese dataset. We again compare vanilla Transformer with the ones without decoder-encoder and decoder-decoder residual connection. While tuning the models, we vary the number of layers from 1 to 4, the learning rate, warmup steps and the number of heads. Specifications of the models, experimental setup, datasets and sample outputs can be found in Sec.~\ref{sec:aexp} in the Appendix.

\begin{table}[t]
	\small{\centering
		\begin{tabular}{p{10em}P{5em}P{4em}}
			\toprule
			\textbf{Model} & \textbf{De-En}& \textbf{En-Vi}\\
			\midrule
			\small{Vanilla Transformers}  & 32.9 & 28.8 \\
			- \small{Dec-Enc Residual} & 24.1 & 21.8 \\
			- \small{Dec-Dec Residual} & 30.6 & 27.2\\
			
			\bottomrule
		\end{tabular}
		\caption{\label{tab:data_eval}BLEU scores ($\uparrow$) for translation task. Please see  Section \ref{sec:exp} for details. }
	}
\end{table}

\textbf{Results} on the effect of residual connections on synthetic tasks can be found in Table \ref{tab:data_synthetic}. As per our hypothesis, all the variants are able to perfectly perform the copy task. For the counting task, the one without decoder-encoder residual connection is incapable of performing it. However, the other two including the one without decoder-decoder residual connection are able to accomplish the task by learning to make decisions based on their prior predictions. Table \ref{tab:sample_count} provides some illustrative sample outputs of the models. For the MT task, results can be found in Table \ref{tab:data_eval}. While the drop from removing decoder-encoder residual connection is significant, it is still able to perform reasonably well since the task can be largely fulfilled by mapping different words from one sentence to another.


For positional masking, our proof technique suggests that due to lack of positional encodings, the model must come up with its own mechanism to make order related decisions. Our hypothesis is that, if it is able to develop such a mechanism, it should be able to generalize to higher lengths and not overfit on the data it is provided. To evaluate this claim, we simply extend the copy task upto higher lengths. The training set remains the same as before, containing sentences of length 5-12 words. We create 5 different validation
sets each containing 1k sentences each. The	first set contains sentences within the same length	as seen in training (5-12 words), the second set contains sentences of length 13-15 words while the third, fourth and fifth sets contain sentences of lengths 15-20, 21-25 and 26-30 words respectively. We consider two models, one which is provided absolute positional encodings and one where only positional masking is applied. Figure \ref{fig:gen_syn} shows the performance of these models across various lengths. The model with positional masking clearly generalizes up to higher lengths although its performance too degrades at extreme lengths. We found that the model with absolute positional encodings during training overfits on the fact that the 13th token is always the terminal symbol. Hence, when evaluated on higher lengths it never produces a sentence of length greater than 12. Other encoding schemes such as relative positional encodings \cite{shaw-etal-2018-self,dai-etal-2019-transformer} can generalize better, since they are inherently designed to address this particular issue. However, our goal is not to propose masking as a replacement of positional encodings, rather it is to determine whether the mechanism that the model develops during training is helpful in generalizing to higher lengths. Note that, positional masking was not devised by keeping generalization or any other benefit in mind. Our claim is only that, the use of masking does not limit the model's expressiveness and it may benefit in other ways, but during practice one should explore each of the mechanisms and even a combination of both. \citet{yang2019assessing} showed that a combination of both masking and encodings is better able to learn order information as compared to explicit encodings.

\begin{figure}[t]
	\centering
	\includegraphics[scale=0.3]{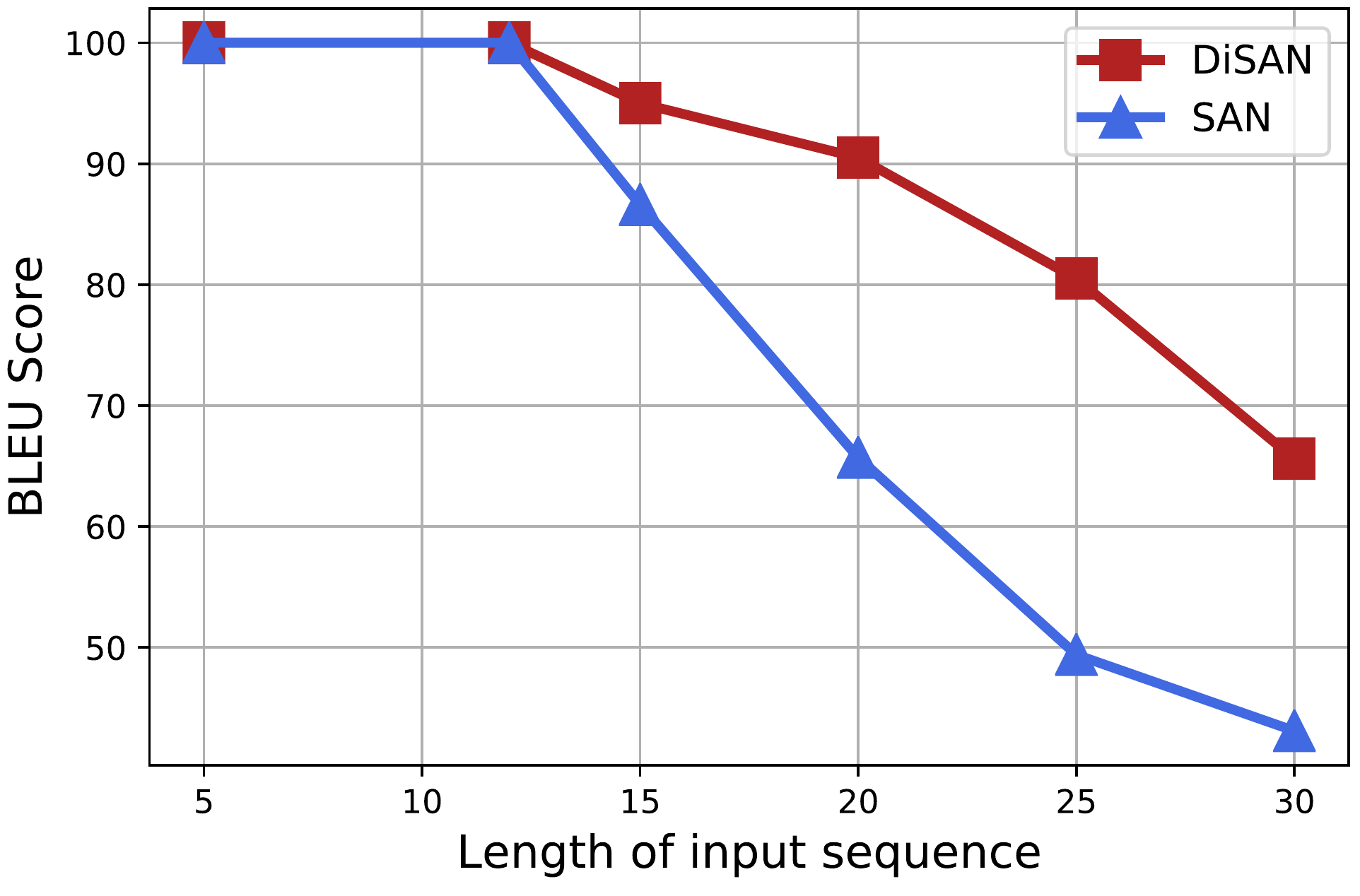}
	\caption{\label{fig:gen_syn} Performance of the two models on the copy task across varying lengths of test inputs. DiSAN refers to Transformer with only positional masking. SAN refers to vanilla Transformers.}
\end{figure}

\begin{table}[th]
	\scriptsize{\centering
		\begin{tabular}{m{13em}m{13em}}
			\toprule
			\textsc{Source} & -- 42  \\
			\textsc{Reference} & -- 43 44 45 46 47 \\
			\midrule
			{\textsc{Vanilla Transformer}} & -- 43 44 45 46 47\\
			{\textsc{ - Dec-Enc Residual}} & -- 27 27 27 27 27\\
			{\textsc{ - Dec-Dec Residual}} & -- 43 44 45 46 47\\
			\bottomrule
		\end{tabular}
		\caption{\label{tab:sample_count} Sample outputs by the models on the counting task. Without the residual connection around Decoder-Encoder block, the model is incapable of predicting more than one distinct output.}
	}
\end{table}
%

\section{Discussion and Final Remarks}

We showed that the class of languages recognized by Transformers and RNNs are exactly the same. This implies that the difference in performance of both the networks across different tasks can be attributed only to their learning abilities. In contrast to RNNs, Transformers are composed of multiple components which are not essential for their computational expressiveness. However, in practice they may play a crucial role. Recently, \citet{voita-etal-2019-analyzing} showed that the decoder-decoder attention heads in the lower layers of the decoder do play a significant role in the NMT task and suggest that they may be helping in language modeling. This indicates that components which are not essential for the computational power may play a vital role in improving the learning and generalization ability.

 \textbf{Take-Home Messages.} We showed that the order information can be provided either in the form of explicit encodings or masking without affecting computational power of Transformers. The decoder-encoder attention block plays a necessary role in conditioning the computation on the input sequence while the residual connection around it is necessary to keep track of previous computations. The feedforward network in the decoder is the only component capable of performing computations based on the input and prior computations. Our experimental results show that removing components essential for computational power inhibit the model's ability to perform certain tasks. At the same time, the components which do not play a role in the computational power may be vital to the learning ability of the network. 

Although our proofs rely on arbitrary precision, which is common practice while studying the computational power of neural networks in theory \cite{siegelmann1992computational,perez2019turing,hahn2019theoretical,yun2019transformers}, implementations in practice work over fixed precision settings. However, our construction provides a starting point to analyze Transformers under finite precision. Since RNNs can recognize all regular languages in finite precision \cite{korsky2019computational}, it follows from our construction that Transformer can also recognize a large class of regular languages in finite precision. At the same time, it does not imply that it can recognize all regular languages given the limitation due to the precision required to encode positional information. We leave the study of Transformers in finite precision for future work.

\section*{Acknowledgements}
We thank the anonymous reviewers for their constructive comments and suggestions. We would also like to thank our colleagues at Microsoft Research and Michael Hahn for their valuable feedback and helpful discussions.

\bibliography{conll}

\begin{thebibliography}{30}
\expandafter\ifx\csname natexlab\endcsname\relax\def\natexlab#1{#1}\fi

\bibitem[{Chen et~al.(2018)Chen, Gilroy, Maletti, May, and
  Knight}]{chen2017recurrent}
Yining Chen, Sorcha Gilroy, Andreas Maletti, Jonathan May, and Kevin Knight.
  2018.
\newblock \href {https://doi.org/10.18653/v1/N18-1205} {Recurrent neural
  networks as weighted language recognizers}.
\newblock In \emph{Proceedings of the 2018 Conference of the North {A}merican
  Chapter of the Association for Computational Linguistics: Human Language
  Technologies, Volume 1 (Long Papers)}, pages 2261--2271, New Orleans,
  Louisiana. Association for Computational Linguistics.

\bibitem[{Dai et~al.(2019)Dai, Yang, Yang, Carbonell, Le, and
  Salakhutdinov}]{dai-etal-2019-transformer}
Zihang Dai, Zhilin Yang, Yiming Yang, Jaime Carbonell, Quoc Le, and Ruslan
  Salakhutdinov. 2019.
\newblock \href {https://doi.org/10.18653/v1/P19-1285} {Transformer-{XL}:
  Attentive language models beyond a fixed-length context}.
\newblock In \emph{Proceedings of the 57th Annual Meeting of the Association
  for Computational Linguistics}, pages 2978--2988, Florence, Italy.
  Association for Computational Linguistics.

\bibitem[{Devlin et~al.(2019)Devlin, Chang, Lee, and
  Toutanova}]{devlin-etal-2019-bert}
Jacob Devlin, Ming-Wei Chang, Kenton Lee, and Kristina Toutanova. 2019.
\newblock \href {https://doi.org/10.18653/v1/N19-1423} {{BERT}: Pre-training of
  deep bidirectional transformers for language understanding}.
\newblock In \emph{Proceedings of the 2019 Conference of the North {A}merican
  Chapter of the Association for Computational Linguistics: Human Language
  Technologies, Volume 1 (Long and Short Papers)}, pages 4171--4186,
  Minneapolis, Minnesota. Association for Computational Linguistics.

\bibitem[{Hahn(2020)}]{hahn2019theoretical}
Michael Hahn. 2020.
\newblock \href {https://doi.org/10.1162/tacl\_a\_00306} {Theoretical
  limitations of self-attention in neural sequence models}.
\newblock \emph{Transactions of the Association for Computational Linguistics},
  8:156--171.

\bibitem[{Hron et~al.(2020)Hron, Bahri, Sohl-Dickstein, and
  Novak}]{hron2020infinite}
Jiri Hron, Yasaman Bahri, Jascha Sohl-Dickstein, and Roman Novak. 2020.
\newblock Infinite attention: Nngp and ntk for deep attention networks.
\newblock \emph{arXiv preprint arXiv:2006.10540}.

\bibitem[{Huang et~al.(2018)Huang, Vaswani, Uszkoreit, Shazeer, Hawthorne, Dai,
  Hoffman, and Eck}]{Huang2018AnIR}
Cheng-Zhi~Anna Huang, Ashish Vaswani, Jakob Uszkoreit, Noam Shazeer, Curtis
  Hawthorne, Andrew~M. Dai, Matthew~D. Hoffman, and Douglas Eck. 2018.
\newblock An improved relative self-attention mechanism for transformer with
  application to music generation.
\newblock \emph{ArXiv}, abs/1809.04281.

\bibitem[{Kim et~al.(2020)Kim, Papamakarios, and Mnih}]{kim2020lipschitz}
Hyunjik Kim, George Papamakarios, and Andriy Mnih. 2020.
\newblock The lipschitz constant of self-attention.
\newblock \emph{arXiv preprint arXiv:2006.04710}.

\bibitem[{Klein et~al.(2017)Klein, Kim, Deng, Senellart, and
  Rush}]{DBLP:journals/corr/KleinKDSR17}
Guillaume Klein, Yoon Kim, Yuntian Deng, Jean Senellart, and Alexander Rush.
  2017.
\newblock \href {https://www.aclweb.org/anthology/P17-4012} {{O}pen{NMT}:
  Open-source toolkit for neural machine translation}.
\newblock In \emph{Proceedings of {ACL} 2017, System Demonstrations}, pages
  67--72, Vancouver, Canada. Association for Computational Linguistics.

\bibitem[{Kolen and Kremer(2001)}]{kolen2001field}
John~F Kolen and Stefan~C Kremer. 2001.
\newblock \emph{A field guide to dynamical recurrent networks}.
\newblock John Wiley \& Sons.

\bibitem[{Korsky and Berwick(2019)}]{korsky2019computational}
Samuel~A Korsky and Robert~C Berwick. 2019.
\newblock On the computational power of rnns.
\newblock \emph{arXiv preprint arXiv:1906.06349}.

\bibitem[{Levine et~al.(2020)Levine, Wies, Sharir, Bata, and
  Shashua}]{levine2020limits}
Yoav Levine, Noam Wies, Or~Sharir, Hofit Bata, and Amnon Shashua. 2020.
\newblock Limits to depth efficiencies of self-attention.
\newblock \emph{arXiv preprint arXiv:2006.12467}.

\bibitem[{Luong and Manning(2015)}]{luong2015stanford}
Minh-Thang Luong and Christopher~D Manning. 2015.
\newblock Stanford neural machine translation systems for spoken language
  domains.
\newblock In \emph{Proceedings of the International Workshop on Spoken Language
  Translation}, pages 76--79.

\bibitem[{McCulloch and Pitts(1943)}]{mcculloch1943logical}
Warren~S McCulloch and Walter Pitts. 1943.
\newblock A logical calculus of the ideas immanent in nervous activity.
\newblock \emph{The bulletin of mathematical biophysics}, 5(4):115--133.

\bibitem[{Merrill et~al.(2020)Merrill, Weiss, Goldberg, Schwartz, Smith, and
  Yahav}]{merrill-etal-2020-formal}
William Merrill, Gail Weiss, Yoav Goldberg, Roy Schwartz, Noah~A. Smith, and
  Eran Yahav. 2020.
\newblock \href {https://www.aclweb.org/anthology/2020.acl-main.43} {A formal
  hierarchy of {RNN} architectures}.
\newblock In \emph{Proceedings of the 58th Annual Meeting of the Association
  for Computational Linguistics}, pages 443--459, Online. Association for
  Computational Linguistics.

\bibitem[{Ott et~al.(2018)Ott, Edunov, Grangier, and
  Auli}]{ott-etal-2018-scaling}
Myle Ott, Sergey Edunov, David Grangier, and Michael Auli. 2018.
\newblock \href {https://doi.org/10.18653/v1/W18-6301} {Scaling neural machine
  translation}.
\newblock In \emph{Proceedings of the Third Conference on Machine Translation:
  Research Papers}, pages 1--9, Brussels, Belgium. Association for
  Computational Linguistics.

\bibitem[{P{\'e}rez et~al.(2019)P{\'e}rez, Marinkovi{\'c}, and
  Barcel{\'o}}]{perez2019turing}
Jorge P{\'e}rez, Javier Marinkovi{\'c}, and Pablo Barcel{\'o}. 2019.
\newblock \href {https://openreview.net/forum?id=HyGBdo0qFm} {On the turing
  completeness of modern neural network architectures}.
\newblock In \emph{International Conference on Learning Representations}.

\bibitem[{Radford et~al.(2018)Radford, Narasimhan, Salimans, and
  Sutskever}]{radford2018improving}
Alec Radford, Karthik Narasimhan, Tim Salimans, and Ilya Sutskever. 2018.
\newblock Improving language understanding by generative pre-training.
\newblock \emph{URL https://s3-us-west-2. amazonaws.
  com/openai-assets/researchcovers/languageunsupervised/language understanding
  paper. pdf}.

\bibitem[{Rush(2018)}]{rush-2018-annotated}
Alexander Rush. 2018.
\newblock \href {https://doi.org/10.18653/v1/W18-2509} {The annotated
  transformer}.
\newblock In \emph{Proceedings of Workshop for {NLP} Open Source Software
  ({NLP}-{OSS})}, pages 52--60, Melbourne, Australia. Association for
  Computational Linguistics.

\bibitem[{Sennhauser and Berwick(2018)}]{sennhauser-berwick-2018-evaluating}
Luzi Sennhauser and Robert Berwick. 2018.
\newblock \href {https://doi.org/10.18653/v1/W18-5414} {Evaluating the ability
  of {LSTM}s to learn context-free grammars}.
\newblock In \emph{Proceedings of the 2018 {EMNLP} Workshop {B}lackbox{NLP}:
  Analyzing and Interpreting Neural Networks for {NLP}}, pages 115--124,
  Brussels, Belgium. Association for Computational Linguistics.

\bibitem[{Shaw et~al.(2018)Shaw, Uszkoreit, and Vaswani}]{shaw-etal-2018-self}
Peter Shaw, Jakob Uszkoreit, and Ashish Vaswani. 2018.
\newblock \href {https://doi.org/10.18653/v1/N18-2074} {Self-attention with
  relative position representations}.
\newblock In \emph{Proceedings of the 2018 Conference of the North {A}merican
  Chapter of the Association for Computational Linguistics: Human Language
  Technologies, Volume 2 (Short Papers)}, pages 464--468, New Orleans,
  Louisiana. Association for Computational Linguistics.

\bibitem[{Shen et~al.(2018)Shen, Zhou, Long, Jiang, Pan, and
  Zhang}]{shen2018disan}
Tao Shen, Tianyi Zhou, Guodong Long, Jing Jiang, Shirui Pan, and Chengqi Zhang.
  2018.
\newblock Disan: Directional self-attention network for rnn/cnn-free language
  understanding.
\newblock In \emph{Thirty-Second AAAI Conference on Artificial Intelligence}.

\bibitem[{Siegelmann(2012)}]{siegelmann2012neural}
Hava~T Siegelmann. 2012.
\newblock \emph{Neural networks and analog computation: beyond the Turing
  limit}.
\newblock Springer Science \& Business Media.

\bibitem[{Siegelmann and Sontag(1992)}]{siegelmann1992computational}
Hava~T Siegelmann and Eduardo~D Sontag. 1992.
\newblock On the computational power of neural nets.
\newblock In \emph{Proceedings of the fifth annual workshop on Computational
  learning theory}, pages 440--449. ACM.

\bibitem[{Skachkova et~al.(2018)Skachkova, Trost, and
  Klakow}]{skachkova-etal-2018-closing}
Natalia Skachkova, Thomas Trost, and Dietrich Klakow. 2018.
\newblock \href {https://doi.org/10.18653/v1/W18-5425} {Closing brackets with
  recurrent neural networks}.
\newblock In \emph{Proceedings of the 2018 {EMNLP} Workshop {B}lackbox{NLP}:
  Analyzing and Interpreting Neural Networks for {NLP}}, pages 232--239,
  Brussels, Belgium. Association for Computational Linguistics.

\bibitem[{Tsai et~al.(2019)Tsai, Bai, Yamada, Morency, and
  Salakhutdinov}]{kerneltsai2019transformer}
Yao-Hung~Hubert Tsai, Shaojie Bai, Makoto Yamada, Louis-Philippe Morency, and
  Ruslan Salakhutdinov. 2019.
\newblock \href {https://doi.org/10.18653/v1/D19-1443} {Transformer dissection:
  An unified understanding for transformer{'}s attention via the lens of
  kernel}.
\newblock In \emph{Proceedings of the 2019 Conference on Empirical Methods in
  Natural Language Processing and the 9th International Joint Conference on
  Natural Language Processing (EMNLP-IJCNLP)}, pages 4344--4353, Hong Kong,
  China. Association for Computational Linguistics.

\bibitem[{Vaswani et~al.(2017)Vaswani, Shazeer, Parmar, Uszkoreit, Jones,
  Gomez, Kaiser, and Polosukhin}]{vaswani2017attention}
Ashish Vaswani, Noam Shazeer, Niki Parmar, Jakob Uszkoreit, Llion Jones,
  Aidan~N Gomez, {\L}ukasz Kaiser, and Illia Polosukhin. 2017.
\newblock Attention is all you need.
\newblock In \emph{Advances in neural information processing systems}, pages
  5998--6008.

\bibitem[{Voita et~al.(2019)Voita, Talbot, Moiseev, Sennrich, and
  Titov}]{voita-etal-2019-analyzing}
Elena Voita, David Talbot, Fedor Moiseev, Rico Sennrich, and Ivan Titov. 2019.
\newblock \href {https://doi.org/10.18653/v1/P19-1580} {Analyzing multi-head
  self-attention: Specialized heads do the heavy lifting, the rest can be
  pruned}.
\newblock In \emph{Proceedings of the 57th Annual Meeting of the Association
  for Computational Linguistics}, pages 5797--5808, Florence, Italy.
  Association for Computational Linguistics.

\bibitem[{Weiss et~al.(2018)Weiss, Goldberg, and Yahav}]{weiss2018practical}
Gail Weiss, Yoav Goldberg, and Eran Yahav. 2018.
\newblock \href {https://doi.org/10.18653/v1/P18-2117} {On the practical
  computational power of finite precision {RNN}s for language recognition}.
\newblock In \emph{Proceedings of the 56th Annual Meeting of the Association
  for Computational Linguistics (Volume 2: Short Papers)}, pages 740--745,
  Melbourne, Australia. Association for Computational Linguistics.

\bibitem[{Yang et~al.(2019)Yang, Wang, Wong, Chao, and Tu}]{yang2019assessing}
Baosong Yang, Longyue Wang, Derek~F. Wong, Lidia~S. Chao, and Zhaopeng Tu.
  2019.
\newblock \href {https://doi.org/10.18653/v1/P19-1354} {Assessing the ability
  of self-attention networks to learn word order}.
\newblock In \emph{Proceedings of the 57th Annual Meeting of the Association
  for Computational Linguistics}, pages 3635--3644, Florence, Italy.
  Association for Computational Linguistics.

\bibitem[{Yun et~al.(2020)Yun, Bhojanapalli, Rawat, Reddi, and
  Kumar}]{yun2019transformers}
Chulhee Yun, Srinadh Bhojanapalli, Ankit~Singh Rawat, Sashank Reddi, and Sanjiv
  Kumar. 2020.
\newblock \href {https://openreview.net/forum?id=ByxRM0Ntvr} {Are transformers
  universal approximators of sequence-to-sequence functions?}
\newblock In \emph{International Conference on Learning Representations}.

\end{thebibliography}
\bibliographystyle{acl_natbib}

\clearpage

\newpage
\appendix

\section{Roadmap}
We begin with various definitions and results. We define simulation of Turing machines by RNNs and state the Turing-completeness result for RNNs. 
We define vanilla and directional Transformers and what it means for Transformers to simulate RNNs. Many of the definitions from the main paper are reproduced here, but in more detail. In Sec.~\ref{subsec:residual} we discuss
the effect of removing a residual connection on computational power of Transformers. Sec.~\ref{subsec:trans_tc} contains the proof of Turing completeness of vanilla
Transformers and Sec.~\ref{sec:directional} the corresponding proof for directional Transformers. Finally, Sec.~\ref{sec:exp} has further details of experiments. 

\section{Definitions}
Denote the set $\{1, 2, \ldots, n\}$ by $[n]$. 
Functions defined for scalars are extended to vectors in the natural way: for a function $F$ defined on a set $A$, for a sequence $(a_1, \ldots, a_n)$ of elements in $A$, we set $F(a_1, \ldots, a_n) := (F(a_1), \ldots, F(a_n))$. Indicator $\indicator(P)$ is $1$, if predicate $P$ is true and is $0$ otherwise. 
For a sequence $\mX = (\vx_{n'},\ldots, \vx_n)$ for some $n' \geq 0$, we set $\mX_j := (\vx_{n'},\ldots, \vx_j)$ for $j \in \{n', i+1, \ldots, n\}$.
We will work with an alphabet $\Sigma = \{\symb_1, \ldots, \symb_\alphabetsize\}$, with $\symb_1 = \#$ and $\symb_\alphabetsize = \$ $.
The special symbols $\#$ and $\$$ correspond to the beginning and end of the input sequence, resp. 
For a vector $\vv$, by $\vzero_\vv$ 
we mean the all-$0$ vector of the same dimension as $\vv$. 
Let $\bt := \min\{t, n\}$

\subsection{RNNs and Turing-completeness}\label{subsec:ssrnn}
Here we summarize, somewhat informally, the Turing-completeness result for RNNs due to \cite{siegelmann1992computational}. 
We recall basic notions from computability theory.
In the main paper, for simplicity we stated the results for \emph{total recursive} functions $\phi: \{0, 1\}^* \to \{0, 1\}^*$, i.e. a function that 
is defined on every $s \in \{0, 1\}^*$ and whose values can be computed by a Turing machine. While total recursive functions form a satisfactory formalization of seq-to-seq tasks, here we state the more general result for 
\emph{partial recursive functions}. Let $\phi: \{0, 1\}^* \to \{0, 1\}^*$ be partial recursive. A partial recursive function is one that need not be defined
for every $s \in \{0, 1\}^*$, and there exists a Turing Machine $\mathcal{M}$ with the following property. 
The input $s$ is initially written on the tape of the Turing Machine $\mathcal{M}$ and the output $\phi(s)$ is the content of the tape upon acceptance which is indicated by halting in a designated accept state. On $s$ for which $\phi$ is undefined, $\mathcal{M}$ does not halt.

We now specify how Turing machine $\mathcal{M}$ is simulated by RNN $R(\mathcal{M})$. 
In the RNNs in \cite{siegelmann1992computational} 
the hidden state $\vh_{t}$ has the form
\[
\begin{array}{rcllr}
\vh_t & = && [  \vq_t, \stack_1, \stack_2 ], 
\end{array}
\]
where $\vq_{t} = [q_1, \ldots, q_s]$ denotes the state of $\mathcal{M}$ one-hot form. Numbers $\stack_1, \stack_2 \in \mbQ$, called stacks, 
store the contents of the tape in a certain Cantor set like encoding (which is similar to, but slightly more involved, than binary representation)
at each step. 
The simulating RNN $R(\mathcal{M})$, gets as input encodings of $s_1 s_2...s_n$ in the first $n$ steps, and from then on receives the vector $\vzero$ as input in each step. If $\phi$ is defined on $s$, then $\mathcal{M}$ halts and accepts with the output $\phi(s)$ the content of the tape. 
In this case, $R(\mathcal{M})$ enters a special accept state, and $\stack_1$ encodes $\phi(s)$ and $\stack_2 = 0$. If $\mathcal{M}$ does not halt then
$R(\mathcal{M})$ also does not enter the accept state. 

\citet{siegelmann1992computational} further show that from $R(\mathcal{M})$ one can further explicitly produce the $\phi(s)$ as its output. In the present paper, we will not deal with explicit production of the output but rather work with the definition of simulation in the previous paragraph. This is
for simplicity of exposition, and the main ideas are already contained in our results. 
 If the Turing machine computes $\phi(s)$ in time $T(s)$, the simulation takes $O(|s|)$ time to encode the input sequence $s$ and $4T(s)$ to compute $\phi(s)$.

\begin{theorem}[\cite{siegelmann1992computational}]\label{th:TC_RNN}
	Given any partial recursive function $\phi: \{0,1\}^{*} \to \{0,1\}^{*}$ computed by Turing machine $\mathcal{M}_\phi$, 
	there exists a simulating RNN $R(\mathcal{M}_\phi)$. 
\end{theorem}

In view of the above theorem, for establishing Turing-completeness of Transformers, it suffices to show that RNNs can be simulated by Transformers. 
Thus, in the sequel we will only talk about simulating RNNs.

\subsection{Vanilla Transformer Architecture}
Here we describe the original transformer architecture due to \cite{vaswani2017attention} as formalized by \cite{perez2019turing}. While our notation and definitions largely follow \cite{perez2019turing}, they are not identical. The transformer here makes use of positional encoding; later we will discuss
the transformer variant using directional attention but without using positional encoding. 

The transformer, denoted $\Trans$, is a sequence-to-sequence architecture. Its input consists of (i) a sequence $\mX = (\vx_1,\ldots, \vx_n)$  of vectors in $\bq^d$, (ii) a seed vector $\vy_0 \in \bq^d$.  
The output is a sequence $\mY=(\vy_1,\ldots, \vy_r)$ of vectors in $\bq^d$. 
The sequence $\mX$ is obtained from the sequence $(s_0, \ldots, s_n) \in  \Sigma^{n+1}$ of symbols by using the embedding mentioned earlier: 
$\vx_i = f(\BaseEmbedding(s_i), \pos(i))$ for $0 \leq i \leq n$.
The transformer consists of composition of \emph{transformer encoder} and a \emph{transformer decoder}. The transformer encoder is obtained by composing one or more \emph{single-layer encoders} and similarly the transformer decoder is obtained by composing one or more \emph{single-layer decoders}. 
For the feed-forward networks in the transformer layers we use the activation as in \cite{siegelmann1992computational}, namely the saturated linear activation function:
\begin{equation}
\sigmoid(x) = \left\{  
\begin{array}{cc}
0 & \quad \text{if  } x < 0, \\
x & \quad \text{if  } 0 \leq x \leq 1, \\
1 & \quad \text{if  } x > 1.
\end{array}
\right.
\end{equation}
As mentioned in the main paper, we can easily work with the standard $\mathsf{ReLU}$ activation via $\sigmoid(x) = \mathsf{ReLU}(x) - \mathsf{ReLU}(x-1)$.
In the following, after defining these components, we will put them together to specify the full transformer architecture. But we begin with self-attention mechanism which is the central feature of the transformer.

\paragraph{Self-attention.} The self-attention mechanism takes as input (i) a \emph{query} vector
$\vq$, (ii) a sequence of \emph{key} vectors $\mK = (\vk_1, \ldots , \vk_n)$, and (iii) a sequence of \emph{value} vectors $\mV = (\vv_1, \ldots , \vv_n)$. All vectors are in $\bq^d$. 

The $\vq$-attention over keys $\mK$ and values $\mV$, denoted by $\Att(\vq, \mK, \mV)$, is a vector $\va$ given by

\begin{eqnarray*}
(\alpha_1, \ldots , \alpha_n) &=& \rho (\fatt(\vq, \vk_1), \ldots , \fatt(\vq, \vk_n)), \\
\va &=& \alpha_1\vv_1 + \alpha_2\vv_2 + \cdots +\alpha_n\vv_n. 
\end{eqnarray*}

The above definition uses two functions $\rho$ and $f^\att$ which we now describe.
For the normalization function $\rho: \mbQ^n \to \mbQ_{\geq 0}^n$ we will use $\hardmax$: for $\vx = (x_1, \ldots, x_n) \in \mbQ^n$, if the maximum value occurs $r$ times among  $x_1, \ldots, x_n$, then 
$\hardmax(\vx)_i := 1/r$ if $x_i$ is a maximum value and $\hardmax(\vx)_i := 0$ otherwise. In practice, the $\softmax$ is often used but its output values are in general not rational. The names soft-attention and hard-attention are used for the attention mechanism depending on which normalization function is used.

For the Turing-completeness proof of vanilla transformers, the scoring function $f^\att$ used is a combination of multiplicative attention \cite{vaswani2017attention} and a non-linear function: $\fatt(\vq, \vk_i) = -\left | \<\vq, \vk_i \> \right |$. For directional transformers, the standard multiplicative attention is used, that is, $\fatt(\vq, \vk_i) =  \< \vq, \vk_i \>$.

\paragraph{Transformer encoder.} 
A \emph{single-layer encoder} is a function $\Enc(\mX; \btheta)$, where $\btheta$ is the parameter vector and 
the input $\mX = (\vx_1, \ldots, \vx_n)$ is a sequence of vector in $\mbQ^d$. The output is another sequence $\mZ = (\vz_1, \ldots, \vz_n)$ of vectors
in $\mbQ^d$. The parameters $\btheta$ specify functions $Q(\cdot), K(\cdot), V(\cdot)$, and $O(\cdot)$, all of type $\mbQ^d \to \mbQ^d$. The functions
$Q(\cdot), K(\cdot),$ and $ V(\cdot)$ are usually linear transformations and this will be the case in our constructions:  
\begin{equation*}
\begin{split}
Q(x_i) = \vx_i^T W_{Q}, \\
K(x_i) = \vx_i^T W_{K}, \\
V(x_i) = \vx_i^T W_{V},
\end{split}
\end{equation*}
where $W_{Q}, W_{K}, W_{V} \in \mbQ^{d \times d}$. The function $O(\cdot)$ is a feed-forward network.
The single-layer 
encoder is then defined by 
\begin{eqnarray}
\va_i &=& \Att(Q(\vx_i), K(\mX), V(\mX)) + \vx_i, \label{eq:enc-enc-att-trans} \\
\vz_i &=& O(\va_i) + \va_i. \nonumber 
\end{eqnarray}
The addition operations $+\vx_i$ and $+\va_i$ are the  residual connections. 
The operation in \eqref{eq:enc-enc-att-trans} is called the encoder-encoder attention block. 

The complete $L$-layer transformer encoder $\TEnc^{\brac{L}}(\mX; \btheta)$ has the same input $\mX = (\vx_1, \ldots, \vx_n)$ as the single-layer encoder. 
By contrast, its output consists of two sequences $(\mK^e, \mV^e)$, each a sequence of $n$ vectors in $\mbQ^d$. The encoder
$\TEnc^{\brac{L}}(\cdot)$ is obtained by repeated application of single-layer encoders, each with its own parameters; and at the end, two trasformation functions $K^L(\cdot)$ and $V^L(\cdot)$ are applied to the sequence of output vectors at the last layer. Functions $K^{\brac{L}}(\cdot)$ and 
$V^{\brac{L}}(\cdot)$ are linear transformations in our constructions. Formally, for $1\leq \ell \leq L-1$ and $\mX^1 := \mX $, we have
\begin{eqnarray*}
\mX^{\ell+1}& =& \Enc(\mX^{\ell}; \btheta_\ell), \\ 
\mK^e &=& K^{\brac{L}}(\mX^L), \\
\mV^e &=& V^{\brac{L}}(\mX^L).    
\end{eqnarray*}
The output of the $L$-layer Transformer encoder $(\mK^e, \mV^e) = \TEnc^{\brac{L}}(\mX)$ is fed to the Transformer decoder which we describe next.


\paragraph{Transformer decoder.} 

%

The input to a \emph{single-layer decoder}  is (i) $(\mK^e, \mV^e)$, the sequences of key and value vectors output by the encoder, and (ii) a sequence $\mY = (\vy_1, \ldots, \vy_k)$ of vectors in $\mbQ^d$. The output is another sequence $\mZ = (\vz_1, \ldots, \vz_k)$ of vectors in $\mbQ^d$. 

Similar to the single-layer encoder, a single-layer decoder is parameterized by functions $Q(\cdot), K(\cdot), V(\cdot)$ and $O(\cdot)$ and is defined by
\begin{eqnarray}
\vp_t & = & \Att(Q(\vy_t),K(\mY_t),V(\mY_t)) 
+ \vy_t, \label{eq:dec-dec}\\
\va_t & = & \Att(\vp_t,\mK^e,\mV^e) 
+ \vp_t, \label{eq:enc-dec} \\
\vz_t & = & O(\va_t) + \va_t. \nonumber
\end{eqnarray}

The operation in \eqref{eq:dec-dec} will be referred to as the \emph{decoder-decoder attention} block and the operation in \eqref{eq:enc-dec} as the
\emph{decoder-encoder attention} block. In the decoder-decoder attention block, positional masking is applied to prevent the network from attending over symbols which are ahead of them.

An $L$-layer Transformer decoder is obtained by repeated application of $L$ single-layer decoders each with its own parameters and a transformation function $F : \bq^d \rightarrow \bq^d$ applied to the last vector in the sequence of vectors output by the final decoder. 
Formally, for $1 \leq \ell \leq L-1$ and $\mY^1 = \mY$ we have
\begin{eqnarray*}
\mY^{\ell+1} &= \Dec((\mK^e, \mV^e), \mY^\ell; \btheta_\ell), \\ 
\vz &= F(\vy_{t}^L). 
\end{eqnarray*}
We use $\vz = \TDec^L((\mK^e, \mV^e), \mY; \btheta)$ to denote an $L$-layer Transformer decoder.
Note that while the output of a single-layer decoder is a sequence of vectors, the output of an $L$-layer Transformer decoder is a single vector.


\paragraph{The complete Transformer.}
A {\em Transformer network} receives an input sequence $\mX$, a seed
vector $\vy_0$, and $r\in \N$. For $t \geq 0$ its output is a sequence $\mY=(\vy_1,\ldots, \vy_r)$ defined by
\begin{equation*}
\tilde{\vy}_{t+1}  =  \TDec\left(\TEnc(\mX), (\vy_0,\vy_1,\ldots,\vy_{t})\right).
\end{equation*}
We get $\vy_{t+1}$ by adding positional encoding: $\vy_{t+1} = \tilde{\vy}_{t+1} + \pos(t+1)$. 
We denote the complete Transformer by $\Trans(\mX,\vy_0) = \mY$. The Transformer ``halts'' when $\vy_T \in H$, where $H$ is a prespecified halting set. 

\paragraph{Simulation of RNNs by Transformers.}
We say that a Transformer simulates an RNN (as defined in Sec.~\ref{subsec:ssrnn}) if on input $s \in \Sigma^*$, at each step $t$, the vector 
$\vy_{t}$ contains the hidden state $\vh_t$ as a subvector: $\vy_{t} = [\vh_t, \cdot]$, and halts at the same step as RNN. 

%
%
%
%
%
%


\section{Results on Vanilla Transformers}\label{sec:vanilla}

\subsection{Residual Connections}\label{subsec:residual}

\begin{proposition}
	The Transformer without residual connection around the Decoder-Encoder Attention block in the Decoder is not Turing Complete
\end{proposition}
\begin{proof}
	Recall that the vectors $\va_t$ is produced from the Encoder-Decoder Attention block in the following way,
	\[
	\va_t  =  \Att(\vp_t,\mK^e,\mV^e) + \vp_t 
	\]

The result follows from the observation that without the residual connections, $\va_t  =  \Att(\vp_t,\mK^e,\mV^e)$, which leads to $\va_t = \sum_{i=1}^{n}\alpha_{i}\vv^{e}_{i}$ for some $\alpha_i$s such that $\sum_i^{n}\alpha_{i}=1$. Since $\vv^{e}_{i}$ is produced from the encoder, the vector $\va_t$ will have no information about its previous hidden state values. Since the previous hidden state information was computed and stored in $\vp_t$, without the residual connection, the information in $\va_t$ depends solely on the output of the encoder. 

One could argue that since the attention weights $\alpha_i$s depend on the query vector $\vp_t$, it could still use it gain the necessary information from the vectors $\vv_i^{e}$s. However, note that by definition of hard attention, the attention weights $\alpha_{i}$ in $\va_t = \sum_{i=1}^{n}\alpha_{i}\vv^{e}_{i}$ can either be zero or some nonzero value depending on the attention logits. Since the attention weights $\alpha_i$ are such that $\sum_i^{n}\alpha_{i}=1$ and all the nonzero weights are equal to each other.  Thus given the constraints there are $2^{n}-1$ ways to attend over $n$ inputs excluding the case where no input is attended over. Hence, the network without decoder-encoder residual connection with $n$ inputs can have at most $2^{n}-1$ distinct $\va_t$ values. This implies that the model will be unable to perform a task that takes $n$ inputs and has to produce more than $2^{n}-1$ outputs. Note that, such a limitation will not exist with a residual connection since the vector $\va_t = \Sigma_{i=1}^{n}\alpha_{i}\vv^{e}_{i} + \vp_t$ can take arbitrary number of values depending on its prior computations in $\vp_t$.

 As an example to illustrate the limitation, consider the following simple problem, given a value $\Delta$, where $0 \leq \Delta \leq 1$, the network must produce the values  $0, \Delta, 2\Delta, \ldots, k \Delta$, where $k$ is the maximum integer such that $k \Delta \leq 1$. If the network receives a single input $\Delta$, the encoder will produce only one particular output vector and regardless of what the value of the query vector $\vp_t$ is, the vector $\va_t$ will be constant at every timestep. Since $\va_{t}$ is fed to feedforward network which maps it to $\vz_t$, the output of the decoder will remain the same at every timestep and it cannot produce distinct values.  If the input is combined with $n-1$ auxiliary symbols (such as $\startsym$ and $\$$), then the network can only produce $2^n-1$ outputs. Hence, the model will be incapable of performing the task if $\Delta < 1/2^n$ .
 
 Thus the model cannot perform the task defined above which RNNs and Vanilla Transformers can easily do with a simple counting mechanism via their recurrent connection.

 For the case of \textbf{multilayer decoder}, consider any $L$ layer decoder model. If the residual connection is removed, the output of decoder-encoder attention block at each layer is  $\va_{t}^{\brac{\ell}} = \sum_{i=1}^{n}\alpha_{i}^{\brac{\ell}}\vv^{e}_{i}$ for $1 \leq \ell \leq L$. Observe, that since output of the decoder-encoder attention block in the last ($L$-th) layer of the decoder is $\va_{t}^{\brac{L}} = \sum_{i=1}^{n}\alpha_{i}^{\brac{L}}\vv^{e}_{i}$. Since the output of the $L$ layer decoder will be a feedforward network over $\va_{t}^{\brac{L}}$, the computation reduces to the single layer decoder case. Hence, similar to the single layer case, if the task requires the network to produce values of $\va_t$ that come from a set with size at least $2^n$, then the network will not be able to perform the task.

This implies that the model without decoder-encoder residual connection is limited in its capability to perform tasks which requires it to make inferences based on previously generated outputs. 

\end{proof}

\subsection{Simulation of RNNs by Transformers with positional encoding}\label{subsec:trans_tc}

\begin{theorem}\label{th:TC_trans}
RNNs can be simulated by vanilla Transformers and hence the class of vanilla Transformers is Turing-complete.
\end{theorem}
\begin{proof}		
	The construction of the simulating transformer is simple: it uses a single head and both the encoder and decoder have one layer. Moreover, the encoder does very little and most of the action happens in the decoder. The main task for the simulation is to design the input embedding (building on the given
	base embedding $\BaseEmbedding$), the feedforward network $O(\cdot)$ and the matrices corresponding to functions $Q(\cdot), K(\cdot), V(\cdot)$.
	
	\paragraph{Input embedding.} The input embedding is obtained by summing the symbol and positional encodings which we next describe. 
	These encodings have dimension $d= 2d_h + d_b + 2 $, where $d_h$ is the dimension of the hidden state of the RNN and $d_b$ is the dimension of 
	the given encoding $\BaseEmbedding$ of the input symbols. 
	We will use the symbol encoding $f^{\mrsymb}:\Sigma \to \mbQ^d$ which is essentially the same as $\BaseEmbedding$ except that the dimension is now larger:
	\begin{equation*}
	f^\mrsymb(s)  =  [  \vzero_{d_h}, f_e(s); \; \vzero_{d_h}, 0, 0 ].
	\end{equation*}
	The positional encoding $\pos: \bn \to \bq^d$ is simply
	\begin{equation*}
	\pos(i) = [\vzero_{d_h}, \vzero_{d_b}, \vzero_{d_h}, i, 1 ].
	\end{equation*}
	Together, these define the combined embedding $f$ for a given input sequence $s_0 s_1\cdots s_n \in \Sigma^{*} $ by
	\begin{equation*}
	f(s_i) = f^\mrsymb(s_i) + \pos(i) = [  \vzero_{d_h}, \BaseEmbedding(s_i), \vzero_{d_h}, i, 1 ].
	\end{equation*}
	The vectors $\vv \in \bq^{d}$ used in the computation of our transformer are of the form
	\begin{equation*}
	\vv  =  [  \vh_1, \vs; \; \vh_2, x_1, x_2 ],
	\end{equation*}		
	%
	where $\vh_1, \vh_2 \in \bq^{d_h}, \vs \in \bq^{d_e}, \text{ and } x_1, x_2 \in \bq$. The coordinates corresponding to the $\vh_i$'s are reserved for computation related to hidden states of the $\RNN$, the coordinates corresponding to $\vs$ are reserved for base embeddings, and those for 
	$x_1$ and $x_2$ are reserved for scalar values related to positional operations. The first two blocks, corresponding to $\vh_1$ and $\vs$
	are reserved for computation of the RNN.

	During the computation of the Transformer, the underlying RNN will get the input $\vs_{\bt}$ at step $t$ for $t = 0, 1, \ldots$, where recall that
	$\bt = \min\{t, n\}.$ This sequence leads to the RNN getting the embedding of the input sequence $s_0, \ldots, s_n$ in the first $n+1$ steps followed by the embedding of the symbol $\$$ for the subsequent steps, which is in accordance with the requirements of \cite{siegelmann1992computational}.
%
%
%
%
	Similar to \cite{perez2019turing} we use the following scoring function in the attention mechanism in our construction,
	
	\begin{equation}\label{eq:att_score_trans}
	f^{\att}(\vq_i, \vk_j) = -|\<\vq_i, \vk_j\>|
	\end{equation}
	
	\paragraph{Construction of $\TEnc$.} As previously mentioned, our transformer encoder has only one layer, and the computation in the encoder is
	very simple: the attention mechanism is not utilized, only the residual connections are. This is done by setting the matrix for $V(\cdot)$ to 
	the all-zeros matrix, and the feedforward networks to always output $\vzero$. The application of appropriately chosen linear transformations
	for the final $K(\cdot)$ and $V(\cdot)$ give the following lemma about the output of the encoder. 
	
	
	\begin{lemma}\label{lem:encoder}
		There exists a single layer encoder denoted by $\TEnc$ that takes as input the sequence $(\vx_1, \ldots, \vx_n, \$)$ and generates the tuple $(\mK^e, \mV^e)$ where $\mK^e = (\vk_1, \ldots, \vk_n) $ and $\mV^e = (\vv_1, \ldots, \vv_n)$ such that,
		\begin{eqnarray*}
		\vk_i  &=&  [ \vzh, \vzs;\;    \vzh,-1, i ],\\
		\vv_i  &=&  [  \vzh, \vs_i ;\; \vzh, 0, 0 ].
	\end{eqnarray*}
		\end{lemma}

	\paragraph{Construction of $\TDec$.} As in the construction of $\TEnc$, our $\TDec$ has only one layer. Also like $\TEnc$, the decoder-decoder attention block just computes the identity: we set $V^{\brac{1}}(\cdot)= \vzero$ identically, and use the residual connection so that $\vp_t = \vy_t$.
	
	For $t \geq 0$, at the $t$-th step we denote the input to the decoder as $\vy_t = \Tilde{\vy}_t + \pos(t)$. 
	Let $\vh_0 = \vzh \text{ and }\Tilde{\vy}_0 = \vzero$. We will show by induction that at the $t$-th timestep 
	we have

	\begin{equation}\label{eqn:induction}
	\vy_t  =  [  \vh_{t}, \vzs;\;  \vzh,t+1, 1 ].
	\end{equation}
	
	By construction, this is true for $t=0$:
	\begin{equation*}
	\vy_0 =  [  \vzh, \vzs;\;    \vzh,1, 1 ].
	\end{equation*}
	Assuming that it holds for $t$, we show it for $t+1$.
		
	By Lemma~\ref{lem:dec-enc-attention}  
	\begin{equation}
	\Att(\vp_t, \mK^e, \mV^e)  =  [  \vzh, \vv_{\overline{t+1}};\; \vzh,0, 0 ].
	\end{equation}

	Lemma~\ref{lem:dec-enc-attention} basically shows how we retrieve the input $\vs_{\overline{t+1}}$ at the relevant step for further computation in the decoder. It follows that
	\begin{eqnarray*}
		\va_t &= \Att(\vp_t, \mK^e, \mV^e) + \vp_t\\
		&=  [ \vh_{t}, \vs_{\overline{t+1}},    \vzh,t+1, 1 ].
	\end{eqnarray*}	
	In the final block of the decoder, the computation for RNN takes place: 
	
	\begin{lemma}\label{lem:rnn_trans_comp}
		There exists a function $O(\cdot)$ defined by feed-forward network such that,	
		\[
		O(\va_t)  =  [  (\vh_{t+1} - \vh_t), -\vs_{\overline{t+1}},    \vzh, -(t+1), -1 ],
		\]		
		where $\mW_{h}, \mW_{x} \text{ and } \vb$ denote the parameters of the RNN under consideration.
	\end{lemma}
	This leads to
	\begin{equation*}
	\vz_t = O(\va_t) + \va_t = [  \vh_{t+1}, \vzs;\;    \vzh, 0, 0 ] .
	\end{equation*}
\end{proof}

We choose the function $F$ for our decoder to be the identity function, therefore
$\Tilde{\vy}_{t+1} = [  \vh_{t+1}, \vzs;\;    \vzh, 0, 0 ]$, which means 
$\vy_{t+1} = \Tilde{\vy}_{t+1} + \pos(i+1)  =  [  \vh_{t+1}, \vzs;\;  \vzh,t+2, 1 ]$, proving our induction hypothesis.

\subsection{Technical Lemmas}

\begin{proof}[\bf \small Proof of Lemma \ref{lem:encoder}]
We construct a single-layer encoder achieving the desired $\mK^e$ and $\mV^e$. We make use of the residual connections and via trivial self-attention we get that $\vz_i = \vx_i$. More specifically for $i \in [n]$ we have 
\begin{eqnarray*}
    V^{\brac{1}}(\vx_i) = \vzero,\\
    \va_i = \vzero + \vx_i,\\
    O(\va_i) = \vzero, \\
    \vz_i = \vzero + \va_i = \vx_i.
\end{eqnarray*}

$V^{(1)}(\vx_i)=\vzero$ can be achieved by setting the weight matrix as the all-$0$ matrix. Recall that $\vx_i$ is defined as
\[
\begin{array}{rcllr}
    \vx_i  & = & [ & \vzh, \vs_i,  \\
            &&&    \vzh,i, 1 &].
     
\end{array}
\]

We then apply linear transformations in $K(\vz_i) = \vz_{i}\mW_k  $ and $V(\vz_i) = \vz_{i}\mW_v$, where
\[
\mW_k^{T} = \left[\begin{array}{ccccc}
0&0&\cdots &0&0\\
\vdots & &\ddots &\vdots & \vdots\\
0&0&\cdots &0&0\\
\hline
0&0&\cdots &0&1\\
0&0&\cdots &-1&0\\
\end{array}\right],
\]
and $\mW_{k} \in \bq^{d \times d}$, and similarly one can obtain $\vv_i$ by setting the submatrix of $\mW_{v} \in \bq^{d \times d}$ formed by the first $d-2$ rows and columns to the identity matrix, and the rest of the entries to zeros.

\end{proof}

\begin{lemma}\label{lem:dec-enc-attention}
	Let $\vq_{t}\in \Q^{d}$ be a query vector such that $\vq=[\cdot,\ldots,\cdot,t+1,1]$ where $t\in \N$ and `$\cdot$' denotes an arbitrary value.
	Then we have 
	\begin{equation}
	\Att(\vq_t, \mK^e, \mV^e)  =  [  \vzh, \vs_{\overline{t+1}}, \vzh,0, 0 ].
	\end{equation}
\end{lemma}
\begin{proof}
Recall that $\vp_{t} = \vy_t = [\vh_t, 0, \ldots,0, t+1, 1]$ and $\vk_i = [0, 0, \ldots,0, -1, i]$ and hence
\[
\< \vp_t, \vk_i \> = i - (t+1), 
\]
\[
f^{\att}(\vp_t, \vk_i) = -|i-(t+1)|.
\]
Thus, for $i \in [n]$, the scoring function$f^{\att}(\vp_t, \vk_i)$ has the maximum value $0$ at index $i = t+1$ if $t < n$; for $t \geq n$,
the maximum value $t+1-n$ is achieved for $i = n$. Therefore
\[
\Att(\vp_t, \mK^e, \mV^e) = \vs_{\overline{t+1}}.
\]
\end{proof}
\begin{proof}[\bf \small Proof of Lemma \ref{lem:rnn_trans_comp}]
%
%
%
%
Recall that
\[
\begin{array}{rcllr}
    \va_t & = & [ & \vh_{t}, \vs_{\overline{t+1}},  \\
            &&&    \vzh,t+1, 1 &]
\end{array}
\]
Network $O(\va_{t})$ is of the form
\[
O(\va_t) = \rmW_2 \sigmoid(\rmW_1 \va_t + \vb_1),
\]
where $\rmW_i \in \bq^{d \times d}$ and $\vb \in \bq^{d}$ and 
\[
\rmW_1  = 
\begin{array}{c c} &
\begin{array}{c c c c} d_h & \; d_e & \; d_h & \; 2 \\
\end{array}
\\
\begin{array}{c c c c}
     d_h  \\
     d_e \\
     d_h \\
     2
\end{array}
& 
\left[
\begin{array}{c|c|c|c}
 \rmW_h & \rmW_x  & \vzero  &  \vzero   \\ \hline
 \vzero & \rmI & \vzero & \vzero \\ \hline
 \rmI & \vzero & \vzero & \vzero \\ \hline
 \vzero & \vzero & \vzero & \rmI \\
\end{array}
\right] 
\end{array}
\]
and $\vb_1 = [\vb_h, \vzs, \vzh, 0, 0]$. Hence
\begin{eqnarray*}
\sigmoid(\rmW_1 \va_t + \vb_1)  = [  \sigmoid(\rmW_{h}\vh_{t} + \rmW_{x}\vs_{\overline{t+1}} + \vb), \\
\vs_{\overline{t+1}}, \vh_t, t+1, 1 ]
\end{eqnarray*}
%
Next we define $\rmW_2$ by
\[
\rmW_2  = 
\begin{array}{c c} &
\begin{array}{c c c c} d_h & \; d_e & \; d_h & \; 2 \\
\end{array}
\\
\begin{array}{c c c c}
     d_h  \\
     d_e \\
     d_h \\
     2
\end{array}
& 
\left[
\begin{array}{c|c|c|c}
 \rmI & \vzero  & -\rmI  &  \vzero   \\ \hline
 \vzero & -\rmI & \vzero & \vzero \\ \hline
 \vzero & \vzero & \vzero & \vzero \\ \hline
 \vzero & \vzero & \vzero & -\rmI \\
\end{array}
\right]. 
\end{array}
\]
This leads to
\begin{eqnarray*}
    O(\va_t) = \rmW_2 \sigmoid(\rmW_1 \va_t + \vb_1) \\
      =  [ \sigmoid(\rmW_{h}\vh_{t} + \rmW_{x}\vs_{\overline{t+1}} + \vb) - \vh_t, - \vs_{\overline{t+1}},  \\
              \vzh, -(t+1), -1 ],
\end{eqnarray*}
which is what we wanted to prove.
\end{proof}

\section{Completeness of Directional Transformers}\label{sec:directional}


There are a few changes in the architecture of the Transformer to obtain directional Transformer. The first change is that there are no positional encodings and thus the input vector $\vx_i$ only consists of $\vs_i$. Similarly, there are no positional encodings in the decoder inputs and hence $\vy_t = \Tilde{\vy}_t$. The vector $\Tilde{\vy}$ is the output representation produced at the previous step and the first input vector to the decoder $\Tilde{\vy}_0 = \vzero$. Instead of using positional encodings, we apply positional masking to the inputs and outputs of the encoder.

\medskip

Thus the encoder-encoder attention in \eqref{eq:enc-enc-att-trans} is redefined as
\begin{equation*}
    \va_i^{\brac{\ell+1}} = \Att(Q(\vz_i^{\bl}), K(\mZ_i^{\bl}), V(\mZ_i^{\bl})) + \vz_i^{\bl}, 
\end{equation*}
where $\mZ^{\brac{0}} = \mX$. Similarly the decoder-encoder attention in \eqref{eq:enc-dec} is redefined by
\begin{equation*}
\va_t^{\brac{\ell}}  =  \Att(\vp_t^{\bl},\mK^{e}_{t},\mV^{e}_{t}) + \vp_t^{\bl}, 
\end{equation*}
where $\ell$ in $\va_{t}^{\brac{\ell}}$ denotes the layer $\ell$ and we use $\vv^{\brac{\ell, b}}$ to denote any intermediate vector being used in $\ell$-th layer and $b$-th block in cases where the same symbol is used in multiple blocks in the same layer.

\begin{theorem}\label{th:TC_trans}
	RNNs can be simulated by vanilla Transformers and hence the class of vanilla Transformers is Turing-complete.
\end{theorem}
%
\begin{proof}
The Transformer network in this case will be more complex than the construction for the vanilla case. The encoder remains very similar, but the decoder is different and has two layers. 
\paragraph{Embedding.} 
We will construct our Transformer to simulate an RNN of the form given in the definition with the recurrence
$$
\vh_t = g(  \mW_h \vh_{t-1} + \mW_x \vx_t + \vb).
$$
The vectors used in the Transformer layers are of dimension $d= 2d_h + d_e+ 4|\Sigma| + 1 $. Where $d_h$ is the dimension of the hidden state of the $\RNN$ and $d_e$ is the dimension of the input embedding. 

All vector $\vv \in \bq^{d}$ used during the computation of the network are of the form
\begin{equation*}
    \vv  =  [ \vh_1, \vh_2, \vs_1, \oh{s_1}, x_1, \oh{s_2} \oh{s_3}, \oh{s_4} ]
\end{equation*}
where $\vh_i \in \bq^{d_h}, \vs \in \bq^{d_e} \text{ and } x_i \in \bq$. These blocks reserved for different types of objects. The vectors $\vh_i$s are reserved for computation related to hidden states of $\RNN$s, $\vs_i$s are reserved for input embeddings and $x_i$s are reserved for scalar values related to positional operations.

Given an input sequence $s_0s_1 s_2\cdots s_n \in \Sigma^{*} $ where $s_{0} = \startsym $ and $s_n = \$$, we use an embedding function $f : \Sigma \rightarrow \bq^d$ defined as
\[
\begin{array}{rcllr}
     f(s_i) = \vx_i & = & [ & \vzh, \vzh, \vs_i,  \\
            &&&    \oh{s_i}, 0, \vzw, \vzw, \vzw &]  
\end{array}
\]

Unlike \cite{perez2019turing}, we use the dot product as our scoring function  as used in  \citet{vaswani2017attention} in the attention mechanism in our construction,
\begin{equation*}
    f^{\att}(\vq_i, \vk_j) = \<\vq_i, \vk_j\>.
\end{equation*}

For the computation of the Transformer, we also use a vector sequence in $\bq^{|\Sigma|}$ defined by 
\begin{equation*}
\prop_t = \frac{1}{t+1}\sum_{j=0}^{t}\oh{ s_t}, 
\end{equation*}
where $0 \leq t \leq n$. 
The vector $\prop_{t} = (\prop_{t,1}, \ldots, \prop_{t,|\Sigma|})$ contains the proportion of each input symbol till step $t$ for $0 \leq t \leq n$. 
Set $\prop_{-1} = \vzero$. From the defintion of $\prop_t$, it follows that at any step $1 \leq k \leq |\Sigma|$ we have 
\begin{equation}\label{eqn:prop_freq}
\prop_{t,k} = \frac{\freq_{t,k}}{t+1},
\end{equation}
where $\freq_{t,k}$ denotes the number of times the $k$-th symbol $\symb_k$ in $\Sigma$ has appeared till the $t$-th step. 
Note that $\prop_{t,0} = \frac{1}{t+1}$ since the first coordinate corresponds to the proportion of the start symbol $\#$ which appears only once at $t=0$.
Similarly, $\prop_{t,|\Sigma|}=0$ for $0 \leq t <n$ and $\prop_{t, |\Sigma|}=1/(t+1)$ for $t \geq n$, since the end symbol $\termsym$ doesn't appear till the end of the input and it appears only once at $t=n$. 

We define two more sequences of vectors in $\bq^{|\Sigma|}$ for $0 \leq t \leq n$: 
\begin{align*}
    \mDelta_t &= \sigmoid(\prop_{t} - \prop_{t-1}), \\
    \diff_t &= (\mDelta_{t,1}, \ldots, \mDelta_{t,|\Sigma|-1}, 1/2^{t+1}). 
\end{align*}
Here $\mDelta_t$ denotes the difference in the proportion of symbols between the $t$-th and $(t-1)$-th steps, with the applicatin of sigmoid activation. 
In vector $\diff_t$, the last coordinate of $\mDelta_t$ has been replaced with $ 1/2^{t+1}$. The last coordinate in $\prop_{t}$ indicates the proportion of the terminal symbol $\$$ and hence the last value in $\mDelta_t$ denotes the change in proportion of $\$$. We set the last coordinate in $\diff_t$ 
to an exponentially decreasing sequence so that after $n$ steps we always have a nonzero score for the terminal symbol and it is taken as input in the underlying $\RNN$. Different and perhaps simpler choices for the last coordinate of $\diff_t$ may be possible. 
Note that $0 \leq \mDelta_{t,k} \leq 1$ and $0 \leq \diff_{t,k} \leq 1$ for $0 \leq t \leq n$ and $1 \leq k \leq |\Sigma|$.

\paragraph{Construction of $\TEnc$.}
The input to the network $\DTrans_M$ is the sequence $(s_0, s_1, \ldots, s_{n-1}, s_n)$ where $s_0 = \startsym$ and $s_n = \termsym$. Our encoder is a simple single layer network such that $\TEnc(\vx_0,\vx_1, \ldots, \vx_n) = (\mK^e, \mV^e) $ where $\mK^e = (\vke_0, \ldots, \vke_n) $ and $\mV^e = (\vve_0, \ldots, \vve_n)$ such that,
\begin{equation}\label{eqn:kei}
\begin{array}{rcllr}
    \vke_i & = & [ & \vzh, \vzh, \vzs,  \\
            &&&    \oh{s_i}, 0, \vzw, \vzw, \vzw &],     
\end{array}
\end{equation}
\[
\begin{array}{rcllr}
    \vve_i & = & [ & \vzh, \vzh, \vs_i,  \\
            &&&    \vzw, 0,  \vzw, \oh{s_i}, \vzw &]. 
\end{array}
\]
Similar to our construction of the encoder for vanilla transformer (\lemref{lem:encoder}), the above $\mK^e$ and $\mV^e$ can be obtained by making the output of $\Att(\cdot)=0$ by choosing the $V(\cdot)$ to always evaluate to $0$ and similarly for $O(\cdot)$, and using residual connections. 
Then one can produce $\mK^e$ and $\mV^e$ via simple linear transformations using $K(\cdot)$ and $V(\cdot)$. 

\paragraph{Construction of $\TDec$.}
At the $t$-th step we denote the input to the decoder as $\vy_t = \Tilde{\vy}_t$, where $0 \leq t \leq r$, where $r$ is the step where the decoder halts. Let $\vh_{-1} = \vzh$ and $\vh_{0} = \vzh$. We will prove by induction on $t$ that for $0 \leq t \leq r$ we have 
\begin{equation}\label{eqn:directional_induction_hypothesis}
\begin{array}{rcllr}
    \vy_t & = & [ & \vh_{t-1},  \vzh, \vzs,  \\
            &&&   \vzw, \frac{1}{2^t}, \vzw, \vzw, \prop_{\minn{t-1}} &]. 
     
\end{array}
\end{equation}
This is true for $t=0$ by the choice of seed vector:
\[
\begin{array}{rcllr}
    \vy_0 & = & [ & \vzh, \vzh, \vzs,  \\
            &&&   \vzw, 1,  \vzw, \vzw, \vzw &].
     
\end{array}
\]
Assuming the truth of \eqref{eqn:directional_induction_hypothesis} for $t$, we show it for $t+1$.

\paragraph{Layer 1.}
Similar to the construction in \lemref{lem:encoder}, in the decoder-decoder attention block we set $V^{\brac{1}}(\cdot)= \vzd$ and use the residual connections to set $\vp_t^{\brac{1}} = \vy_t$.
\medskip
At the $t$-th step in the decoder-encoder attention block of layer 1 we have
\begin{align*}
    \Att(\vp_{t}^{\brac{1}}, \mK^e_{\bt},  \mV^e_{\bt} ) = \sum_{j=0}^{\bt}\hat{\alpha}_{t, j}^{\brac{1,2}} \vv^e_j, 
\end{align*}
where
\begin{align*}
&(\hat{\alpha}_{t,1}^{\brac{2,2}}, \ldots, \hat{\alpha}_{t,\bt}^{\brac{2,2}}) \\
&= \hardmax\left(\<\vp_{t}^{\brac{1}}, \vk^e_{1}\>, \ldots, \<\vp_{t}^{\brac{1}}, \vk^e_{\bt}\>\right) \\
&= \hardmax(0, \ldots, 0) \\
&= \left(\frac{1}{\bt +1}, \ldots, \frac{1}{\bt +1}\right).
\end{align*}
Therefore
\begin{equation*}
\begin{array}{rcllr}
    \sum_{j=0}^{\bt}\hat{\alpha}_{t, j}^{\brac{1,2}} \vv^e_j & = & [ & \vzh, \vzh, \vs_{0:t},  \\
            &&&   \vzw, 0, \vzw, \prop_{\bt}, \vzw &]
\end{array}
\end{equation*}
where
\[
    \vs_{0:t} = \frac{1}{\bt +1}\sum_{j=0}^{\bt}\vs_{j}.
\]
Thus,
\begin{align*}
       \va_{t}^{\brac{1}} &=  \Att(\vp_{t}^{\brac{1}}, \mK^e_{\bt},  \mV^e_{\bt} )  + \vp_t^{\brac{1}} \nonumber \\
    &= [  \vh_{t-1}, \vzh, \vs_{0:t}, \vzw, \frac{1}{2^t}, \vzw, \prop_{\bt}, \prop_{\minn{t-1}} ]. \nonumber
\end{align*}
In Lemma~\ref{lem:disan-l1-b3} we construct feed-forward network  $O^{\brac{1}}(\cdot)$ such that
\begin{align*}
   O^{\brac{1}}(\va_{t}^{\brac{1}})  
    =   [ \vzh, \vzh, -\vs_{0:t},  
             \diff_{\bt},  -\frac{1}{2^{t}}+ \frac{1}{2^{t+1}},\\ \vzw, 
             \shoveright{-\prop_{\bt},  -\prop_{\minn{t-1}}+\prop_{\bt}].}
\end{align*}
Hence
\begin{align}\label{eqn:z1t}
    \vz_{t}^{\brac{1}} &= O^{\brac{1}}(\va_{t}^{\brac{1}}) + \va_{t}^{\brac{1}} \\
    &=  [  \vh_{t-1},  \vzh ,  \vzs,  
            \diff_{\bt}, \frac{1}{2^{t+1}}, \vzw, \vzw, \prop_{\bt} ].	 \nonumber
\end{align}


\paragraph{Layer 2.}
In the first block of layer 2, we set the value transformation function to identically zero similar to \lemref{lem:encoder}, i.e. $V^{\brac{2}}(\cdot)= \vzero$ which leads to the output of $\Att(\cdot)$ to be $\vzero$ and then using the residual connection we get $\vp_{t}^{\brac{2}} = \vz_{t}^{\brac{1}}$.
It follows by Lemma~\ref{lem:dec-enc-l2-attention} that
\begin{align*}
    \Att&(\vp_t^{\brac{2}}, \mK^e_{\bt}, \mV^e_{\bt})    \\
  & = [  \vzh, \vzh, \vs_{\bt}, 
           \vzw, 0, \vzw, \oh{s_t}, \vzw ].
\end{align*}
Thus,
\begin{align*}
    \va_{t}^{\brac{2}} &=  \Att(\vp_t^{\brac{2}}, \mK^e_{\bt}, \mV^e_{\bt}) + \vp_{t}^{\brac{2}}    \\
   & = [  \vh_{t-1}, \vzh, \vs_{\bt},  
              \diff_{\bt} , \frac{1}{2^{t+1}}, \vzw, \oh{s_t}, \prop_{\bt} ].
\end{align*}
In the final block of the decoder in the second layer, the computation for RNN takes place. 
In Lemma~\ref{lem:rnn_dirtrans_comp} below we construct the feed-forward network $O^{\brac{2}}(\cdot)$ such that 
\begin{align*}
O^{\brac{2}}(\va_t^{\brac{2}})  =  [\sigmoid(\mW_{h}\vh_{t-1} + \mW_{x}\vs_{\bt} + \vb) - \vh_{t-1} \\ 
\shoveright{\vzh, -\vs_{\bt},  -\diff_{t} , 0, \vzw, - \oh{s_t}, \vzw]} 
\end{align*}
and hence
\begin{align*}
    \vz_{t}^{\brac{2}} = & O^{\brac{2}}(\va_t^{\brac{2}}) + \va_t^{\brac{2}} \\
     = & [  \sigmoid(\rmW_{h}\vh_{t-1} + \rmW_{x}\vs_{\bt} + \vb), \vzh, \vzs,  \\
         &  \vzw,  \frac{1}{2^{t+1}}, \vzw, \vzw, \prop_{\bt}   ],     
\end{align*}
which gives
\begin{equation*}
\begin{array}{rcllr}
    \vy_{t+1} & = & [ &  \vh_{t} , \vzh, \vzs,  \\
            &&&    \vzw,  \frac{1}{2^{t+1}}, \vzw, \vzw, \prop_{\bt}   &],
     
\end{array}
\end{equation*}
proving the induction hypothesis \eqref{eqn:directional_induction_hypothesis} for $t+1$, and completing the simulation of RNN.
\end{proof}

\subsection{Technical Lemmas}

\begin{lemma}\label{lem:disan-l1-b3}
	There exists a function $O^{\brac{1}}(.)$ defined by feed-forward network such that,
	\begin{align*}
	O^{\brac{1}}(\va_{t}^{\brac{1}})  &=   [ \vzh, \vzh, -\vs_{0:t}, \diff_t,\\
	&  -\frac{1}{2^{t}}+ \frac{1}{2^{t+1}}, \vzw,\;-\prop_{t},  -\prop_{t-1}+\prop_{t} ]
	\end{align*}
\end{lemma}

\begin{proof}
We define the feed-forward network $O^{(1)}(.)$ such that
\begin{align*}
    O^{\brac{1}}(\va_{t}^{\brac{1}})  &=   [  \vzh, \vzh,  -\vs_{0:t},   \diff_t - \prop_{t}, \\
  &  -\frac{1}{2^{t}}+ \frac{1}{2^{t+1}},\; \vzw,\;\vzw, \;\; -\prop_{t-1}+\prop_{t} ]
\end{align*}
where
\[
\diff_t = (\mDelta_{t,1}, \ldots, \mDelta_{t,n-1}, 1/2^{t+1}), \quad 0 \leq \delta_t \leq 1
\]
\medskip

Recall that,

\[
\begin{array}{rcllr}
    \va_{t}^{\brac{1}} & =  & [ & \vh_{t-1}, \vzh, \vs_{0:t},  \\
            &&&   \prop_{t}, \frac{1}{2^t}, \vzw, \vzw, \prop_{t-1} &]
    \end{array}
\]

We define the feed-forward network $O(\va_{t})$ as follows,

\[
O^{(1)}(\va_t) = \rmW_2 \sigmoid(\rmW_1 \va_t^{(1)} + \vb_1)
\]
where $\mW_i \in \bq^{d \times d}$ and $\vb_1 \in \bq^{d}$. Define $\mW_1$ as
\[
\begin{array}{c c} &
\begin{array}{c c c c c c c} 2d_h &  d_e &  d_{\omega} & 1 & d_{\omega} & d_{\omega} &d_{\omega} \\
\end{array}
\\
\begin{array}{c c c c c c c}
     2d_h  \\
     d_e \\
     d_{\omega}-1 \\
     1 \\
     1\\
     d_{\omega} \\
     d_{\omega} \\
     d_{\omega}
\end{array}
& 
\left[
\begin{array}{c|c|c|c|c|c|c}
 \vzero & \vzero  & \vzero  &  \vzero &  \vzero &  \vzero &  \vzero   \\ \hline
 \vzero & \rmI  & \vzero  &  \vzero &  \vzero &  \vzero &  \vzero   \\ \hline
 \vzero & \vzero  & \vzero  &  \vzero &  \vzero &  \rmI &  -\rmI   \\ \hline
 \vzero & \vzero  & \vzero  &  \frac{1}{2} &  \vzero &  \vzero &  \vzero   \\ \hline
 \vzero & \vzero  & \vzero  &  \frac{1}{2} &  \vzero &  \vzero &  \vzero   \\ \hline
 \vzero & \vzero  & \rmI  &  \vzero &  \vzero &  \vzero &  \vzero   \\ \hline
 \vzero & \vzero  & \vzero  &  \vzero &  \vzero &  \rmI &  \vzero   \\ \hline
 \vzero & \vzero  & \vzero  &  \vzero &  \vzero &  \vzero &  \rmI
\end{array}
\right] 

\end{array}
\]
and $\vb_1 = \vzero$, then
\begin{align*}
 \sigmoid(\mW_1 \va_t^{(1)} + \vb_1)   =   [& \vzh, \vzh, \vs_{0:t},  \mDelta_t, \frac{1}{2^{t+1}}, \\
 & \prop_t,\;\prop_{t-1}, \;\; \prop_{t-1} ]
\end{align*}

We define $\mW_2$ as 

\[ 
\begin{array}{c c} &
\begin{array}{c c c c c c c} 2d_h &  d_e &  d_{\omega-1} & 2 & d_{\omega} & d_{\omega} &d_{\omega} \\
\end{array}
\\
\begin{array}{c c c c c c c}
     2d_h  \\
     d_e \\
     d_{\omega}-1 \\
     1 \\
     1 \\
     d_{\omega} \\
     d_{\omega} \\
     d_{\omega}
\end{array}
& 
\left[
\begin{array}{c|c|c|c|c|c|c}
 \vzero & \vzero  & \vzero  &  \vzero &  \vzero &  \vzero &  \vzero   \\ \hline
 \vzero & -\rmI  & \vzero  &  \vzero &  \vzero &  \vzero &  \vzero   \\ \hline
 \vzero & \vzero  & \rmI  &  \vzero &  \vzero &  \vzero &  \vzero   \\ \hline
 \vzero & \vzero  & \vzero  &  1,0 &  \vzero &  \vzero &  \vzero   \\ \hline
 \vzero & \vzero  & \vzero  &  -2, 1 &  \vzero &  \vzero &  \vzero   \\ \hline
 \vzero & \vzero  & \rmI  &  \vzero &  \vzero &  \vzero &  \vzero   \\ \hline
 \vzero & \vzero  & \vzero  &  \vzero &  \vzero &  -\rmI &  \vzero   \\ \hline
 \vzero & \vzero  & \vzero  &  \vzero &  \vzero &  \rmI &  -\rmI
\end{array}
\right] 

\end{array}
\]

This leads to
\begin{align*}
    O^{\brac{1}}(\va_{t}^{\brac{1}}) & =   [  \vzh, \vzh,  \vs_{0:t},  \diff_t, \\ 
   & -\frac{1}{2^{t}}+ \frac{1}{2^{t+1}},\; \vzw,\;-\prop_t, \;\; -\prop_{t-1}+\prop_{t}]
\end{align*}

which is what we wanted to prove.

\end{proof}

\begin{lemma}\label{lem:dec-enc-l2-attention}
	Let $\vp_{t}^{\brac{2}} \in \Q^{d}$ be a query vector such that
	
	\[
	\begin{array}{rcllr}
	\vp_{t}^{\brac{2}} & =  & [ & \cdot, \;\; \cdot ,  \;\; \cdot,  \\
	&&&   \diff_t, \;\; \cdot, \cdot, \cdot, \cdot &]
	\end{array}
	\]
	where $t \geq 0$  and `$\cdot$' denotes an arbitrary value. Then we have
	\begin{equation} \label{eqn:dec-enc-l2-attention}
	\begin{array}{rcllr}
	\Att(\vp_t^{\brac{2}}, \mK^e_{\bt}, \mV^e_{\bt}) & = & [ & \vzh, \vzh, \vs_{\bt},  \\
	&&&   \vzw, 0, \vzw, \oh{s_t}, \vzw &].
	\end{array}
	\end{equation}
\end{lemma}

\begin{proof}
Let  
	\begin{align*}
	&(\hat{\alpha}_{t,1}^{\brac{2,2}}, \ldots, \hat{\alpha}_{t,\bt}^{\brac{2,2}}) \\
	&= \hardmax\left(\<\vp_{t}^{\brac{2}}, \vk^e_{1}\>, \ldots, \<\vp_{t}^{\brac{2}}, \vk^e_{\bt}\>\right)
	\end{align*}
	be the vector of normalized attention scores in the decoder-encoder attention block of layer 2 at time $t$. Then
	\[
	\Att(\vp_{t}^{\brac{2}}, \mK^e_{\bt},  \mV^e_{\bt} ) = \sum_{j=0}^{\bt}\hat{\alpha}_{t, j}^{\brac{2,2}} \vv^e_j.
	\]
	We claim that 
	\begin{claim}\label{claim:normalized_scores} For $t \geq 0$ we have 
		\begin{align*}
		&(\hat{\alpha}_{t,1}^{\brac{2,2}}, \ldots, \hat{\alpha}_{t,\bt}^{\brac{2,2}}) \\= 
		&\frac{1}{\lambda_{\bt}}\left(\indicator(s_0=s_t), \indicator(s_1=s_t), \ldots, \indicator(s_{\bt}=s_t)\right),
		\end{align*}
%
		where $\lambda_t$ is a normalization factor given by $\lambda_t = \sum_{j=0}^{n-1} \indicator(s_j = s_t)$.
	\end{claim}
	We now prove the lemma assuming the claim above. Denote the L.H.S. in \eqref{eqn:dec-enc-l2-attention} by $\bm{\gamma}_t$. Note that if 
	$s_j = s_t$, then $\vv^e_j = \bm{\gamma}_t$. Now we have
	\begin{eqnarray*}
		\sum_{j=0}^{\bt}\hat{\alpha}_{t, j}^{\brac{2,2}} \vv^e_j &=& \frac{1}{\lambda_t} \sum_{j=0}^{\bt} \indicator(s_j=s_t) \, \vv^e_j \\
		&=& \frac{1}{\lambda_t} \left(\sum_{j=0}^{\bt} \indicator(s_j=s_t)\right) \bm{\gamma}_t  \\
		&=& \bm{\gamma}_t,
	\end{eqnarray*}
	completing the proof of the lemma modulo the proof of the claim, which we prove next. 
\end{proof}

\begin{proof}(of Claim~\ref{claim:normalized_scores})
	%
	%
	For $0 < t \leq n$, the vector $\prop_t - \prop_{t-1}$ has the form
	\[
	\left(\left(\frac{1}{t+1} - \frac{1}{t}\right), \ldots, \left(\frac{\freq_{t,k}}{t+1}-\frac{\freq_{t-1,k}}{t}\right), \ldots , 0\right).
	\]
	
	If  $s_t = \symb_k$, then 
	\begin{eqnarray}
	(\prop_t &-& \prop_{t-1})_k \\
	 &=&  \left(\frac{\freq_{t,k}}{t+1}-\frac{\freq_{t-1,k}}{t}\right) \\
	&=& \left(\frac{\freq_{t-1,k}+1}{t+1}-\frac{\freq_{t-1,k}}{t}\right) \\
	&=& \frac{t-\freq_{t-1,k}}{t(t+1)} \\
	&\geq& \frac{1}{t(t+1)}.
	\end{eqnarray}
	The last inequality used our assumption that $s_0 = \#$ and that $\#$ does not occur at any later time and therefore $\freq_{t-1,j} < t$. 
	On the other hand, if $s_t \neq \symb_k$, then
	\begin{eqnarray}
	(\prop_t - \prop_{t-1})_k &=&  \left(\frac{\freq_{t,k}}{t+1}-\frac{\freq_{t-1,k}}{t}\right) \nonumber \\
	&=& \left(\frac{\freq_{t-1,k}}{t+1}-\frac{\freq_{t-1,k}}{t}\right) \nonumber \\
	&=& -\frac{\freq_{t-1,j}}{t(t+1)}  \\
	&\leq& 0. \nonumber
	\end{eqnarray}
	This leads to,
	\[
	\begin{array}{cc}
	(\prop_t - \prop_{t-1})_k \quad > 0 & \quad \text{ if } s_t = \symb_k, \\
	(\prop_t - \prop_{t-1})_k \quad \leq 0 & \quad \text{otherwise}.	
	\end{array}
	\]
	In words, the change in the proportion of a symbol is positive from step $t-1$ to $t$ if and only if it is the input symbol at the $t$-th step.
	For $0 \leq t \leq n$ and $1 \leq k \leq |\Sigma|$, this leads to
	\[
	\begin{array}{cc}
	\mDelta_{t,k} = \sigmoid(\prop_t - \prop_{t-1})_k \quad > 0 & \quad \text{ if } s_t = \symb_k, \\
	\mDelta_{t,k} = \sigmoid(\prop_t - \prop_{t-1})_k \quad =  0 & \quad \text{otherwise},
	\end{array}
	\]
	For $t > n$,
	\[
	\mDelta_t = \vzero.
	\]

	Recall that $\vp_t^{\brac{2}}= \vz_{t}^{\brac{1}}$ which comes from \eqref{eqn:z1t}, and $\vk^e_j$ is defined in \eqref{eqn:kei}. We reproduce these for convenience:
	\begin{eqnarray*}
		\begin{array}{rcllr}
			\vp_{t}^{\brac{2}} & =  & [ & \vh_{t-1}, \;\; \vzh ,  \;\; \vzs,  \\
			&&&   \diff_{\bt}, \;\; \frac{1}{2^{t+1}}, \vzw, \vzw, \;\;\prop_{\bt} &],
		\end{array} \\
		\begin{array}{rcllr}
			\vk^e_j & = & [ & \vzh, \vzh, \vzs,  \\
			&&&    \oh{s_j}, 0, \vzw, \vzw, \vzw &].
		\end{array}
	\end{eqnarray*}
	%
	It now follows that for $0 < t < n$, if $0 \leq j \leq t$ is such that $s_j \neq s_t$, then 
	\begin{equation*}
	\<\vp_{t}^{\brac{2}}, \vk^e_{j}\> = \<\diff_t, \oh{s_j} \> = \diff_{t, i} = 0. 
	\end{equation*}
	And for $0 < t < n$, if $0 \leq j \leq t$ is such that $s_j = s_t = \symb_i$, then 
	\begin{eqnarray}
	\<\vp_{t}^{\brac{2}}, \vk^e_{j}\> = \<\diff_t, \oh{s_j} \> = \diff_{t, i} \\ = \frac{t-\freq_{t-1,j}}{t(t+1)} \geq \frac{1}{t(t+1)}.
	\end{eqnarray}
	Thus, for $0 \leq t < n$, in the vector $\left(\<\vp_{t}^{\brac{2}}, \vk^e_{0}\>, \ldots, \<\vp_{t}^{\brac{2}}, \vk^e_{t}\> \right)$, the largest coordinates are the ones indexed by $j$ with $s_j = s_t$ and they all equal $\frac{t-\freq_{t-1,i}}{t(t+1)}$. 
	All other coordinates are $0$. For $t \geq n$, only the last coordinate 
	$\<\vp_{t}^{\brac{2}}, \vk^e_{n}\>= \<\diff_t, \oh{\$} \> = \frac{1}{2^{t+1}}$ is non-zero.
	Now the claim follows immediately by the definition of $\hardmax$.
\end{proof}

\begin{lemma}\label{lem:rnn_dirtrans_comp}
	There exists a function $O^{\brac{2}}(.)$ defined by feed-forward network such that, for $t \geq 0$,	
\begin{align*}
	O^{\brac{2}}(\va_t^{\brac{2}})  =  [\sigmoid(\mW_{h}\vh_{t-1} + \mW_{x}\vs_{\bt} + \vb) - \vh_{t-1}, \\ 
\shoveright{\vzh, -\vs_{\bt},  -\diff_{t} , 0, \vzw, - \oh{s_t}, \vzw]} 
\end{align*}	
	where $\mW_{h}, \mW_{x} \text{ and } \vb$ denote the parameters of the RNN under consideration.	
\end{lemma}

\begin{proof}
Proof is very similar to proof of lemma \ref{lem:rnn_trans_comp}.
\end{proof}

\section{Details of Experiments}\label{sec:aexp}
In this section, we describe the specifics of our experimental setup. This includes details about the dataset, models, setup and some sample outputs.
\subsection{Impact of Residual Connections}
The models under consideration are the vanilla Transformer, the one without decoder-encoder residual connection and the one without decoder-decoder residual connection. 
For the synthetic tasks, we implement a single layer encoder-decoder network with only a single attention head in each block. Our implementation of the Transformer is adapted from the implementation of \cite{rush-2018-annotated}. Table \ref{tab:sample_copy} provides some illustrative sample outputs of the models for the copy task.


\begin{table}[th]
	\scriptsize{\centering
		\begin{tabular}{m{6em}m{22em}}
			\toprule
			\textsc{Source \& reference} & -- there was no problem at all says douglas ford chief executive officer of the futures exchange \\
			\midrule
			{\textsc{Directional Transformer}} & -- there was no problem at all says douglas ford chief executive officer of the futures exchange\\
			{\textsc{Vanilla Transformer}} & -- there was no problem at all says douglas ford chief executive officer \\
			\bottomrule
		\end{tabular}
		\caption{\label{tab:sample_copy}Sample outputs by the models on the copy task on length 16. With absolute positional encodings the model overfits on terminal symbol at position 13 and generates sequence of length 12.}
	}
\end{table}

For the machine translation task, we use OpenNMT \cite{DBLP:journals/corr/KleinKDSR17} for our implementation.
For preprocessing the German-English dataset we used the \href{https://github.com/pytorch/fairseq/blob/e734b0fa58fcf02ded15c236289b3bd61c4cffdf/data/prepare-iwslt14.sh}{script} from fairseq. The dataset contains about 153k training sentences, 7k development sentences and 7k test sentences. The hyperparameters to train the vanilla Transformer were obtained from fairseq's \href{https://github.com/pytorch/fairseq/tree/master/examples/translation}{guidelines}. We tuned the parameters on the validation set for the two baseline model. To preprocess the English-Vietnamese dataset, we follow \citet{luong2015stanford}. The dataset contains about 133k training sentences. We use the tst2012 dataset containing 1.5k sentences for validation and tst2013 containing 1.3k sentences as test set. We use noam optimizer in all our experiments. While tuning the network, we vary the number of layer from 1 to 4, the learning rate, the number of heads, the warmup steps, embedding size and feedforward embedding size.

\subsection{Masking and Encodings}

Our implementation for directional transformer is based on \cite{yang2019assessing} but we use only unidirectional masking as opposed to bidirectional used in their setup. While tuning the models, we vary the layers from 1 to 4, the learning rate, warmup steps and the number of heads.

\end{document}